%% file: main.tex
\documentclass{article}
\usepackage{iclr2027_conference,times}
\usepackage[T1]{fontenc}
\usepackage[utf8]{inputenc}
\usepackage{amsmath,amssymb,amsthm}
\usepackage{booktabs,multirow}
\usepackage{graphicx}
\usepackage{xcolor}
\usepackage{enumitem}
\usepackage[ruled,vlined]{algorithm2e}
\usepackage{xspace}
\usepackage{url}
\usepackage[hidelinks]{hyperref}

\theoremstyle{plain}
\newtheorem{theorem}{Theorem}
\newtheorem{lemma}{Lemma}
\theoremstyle{definition}
\newtheorem{definition}{Definition}
\theoremstyle{remark}
\newtheorem{remark}{Remark}

\newcommand{\R}{\mathbb{R}}
\newcommand{\E}{\mathbb{E}}
\newcommand{\Pp}{\mathbb{P}}
\newcommand{\Var}{\mathrm{Var}}
\newcommand{\norm}[1]{\lVert #1 \rVert_2}
\newcommand{\Quant}[1]{\mathcal{Q}_{#1}}
\newcommand{\Blocks}{\mathcal{G}}
\newcommand{\Feas}{\mathcal{F}}
\newcommand{\muG}{\mu_{G}}
\newcommand{\sigG}{\sigma_{G}}

\newcommand{\Dcal}{\mathcal{D}_{\mathrm{cal}}}
\newcommand{\sloc}{s^{\mathrm{loc}}}
\newcommand{\normop}[1]{\lVert #1 \rVert_{\mathrm{op}}}

\newcommand{\cmt}[1]{}

\DeclareRobustCommand{\revdel}[1]{}
\definecolor{reviewadded}{RGB}{0,0,0}
\DeclareRobustCommand{\revadd}[1]{{\color{reviewadded}#1}}

\newenvironment{reviewaddition}{\begingroup\color{reviewadded}}{\endgroup}
\usepackage{etoolbox}
\AtBeginEnvironment{thebibliography}{%
  \let\revieworiginalbibitem\bibitem
  \renewcommand{\bibitem}[2][]{%
    \ifstrequal{#2}{arbelaez2011contour}{\color{reviewadded}}{\color{black}}%
    \revieworiginalbibitem[#1]{#2}}}
\title{Layerwise Error Attribution for Fast and Robust Mixed-Precision Post-Training Quantization}
\author{
	Samy Houache  \\
	Univ. Bordeaux, IMB \\
	Thales AVS, France \\
	\texttt{samy.houache@u-bordeaux.fr }
	\And
	Yann Traonmilin \\
	Univ. Bordeaux, Bordeaux INP, \\ CNRS, IMB, F-33400, Talence, France \\
	\texttt{yann.traonmilin@math.u-bordeaux.fr}
	\And
	Jean-François Aujol \\
	Univ. Bordeaux, Bordeaux INP, \\ CNRS, IMB, F-33400, Talence, France \\
	\texttt{jean-francois.aujol@math.u-bordeaux.fr}
}
	
\date{}

\iclrfinalcopy 

\begin{document}
\maketitle
\lhead{Preprint}

\begin{abstract}
Mixed-precision post-training quantization is a network compression method that assigns bits layer by layer, under a global memory budget using a small calibration set. The main difficulties are to overcome the combinatorial nature of the allocation problem and to manage the sensitivity to small, potentially corrupted databases. Hence, an efficient allocation method should be fast to compute and preserve model quality when calibration data are corrupted.  To design such a method, we derive a layerwise probabilistic analysis of the quantization error that separates propagated error from the local perturbation introduced at a given layer. We use this local term to build a separable score for a simple allocation algorithm, that requires no external solver. The probabilistic nature of our approach brings robustness to corrupted data. On denoising tasks with DRUNet, with an average budget of $4$ bits per weight, our method matches or improves state-of-the-art mixed-precision baselines under clean calibration, and is more robust to corrupted calibration, with  PSNR gains of up to $7.5$~dB under the tested corruptions.  Experiments show  bit-allocation speed-ups from $28\times$ to $2{,}570\times$ over the studied baselines. For quantized diffusion models, our experiments show that a direct application of our framework also improves the state-of-the-art.
\end{abstract}

\input{sections/01_intro_related}
\input{sections/02_setup_theory_method}

\input{sections/03_experiments_revised}

\section{Discussion and conclusion}
\label{sec:discussion}
Mixed-precision PTQ needs sensitive information at each layer in order to evaluate different bit-width configurations. Our layerwise analysis keeps track of the origin of the quantization error  by separating what is carried from previous layers from what is introduced locally. Our practical algorithm turns the local perturbation into a reusable table of  scores for each block and bit-width, and then meets the budget with a simple greedy algorithm and no external solver. This addresses the two practical requirements set out in the introduction: fast allocation and robustness to corrupted calibration data.
On DRUNet, our method combines competitive quality under clean calibration with gains of up to $7.5$~dB under the tested corruptions and bit-width allocation speed-ups from $28\times$ to $2{,}570\times$. On CelebA-HQ, our  mixed-precision model combines competitive FID with the highest mean PSNR to FP32 in the paired comparison. For future work, a major issue is to implement our quantized models on embedded platforms, facing the limitations of current software to manage mixed-precision models. From a theoretical perspective, the question of improving our layerwise probabilistic bounds remains an open question.

\subsection*{AI use statement}
We used generative AI tools to improve the English and to find related papers and public code for mixed-precision comparisons. All AI-assisted suggestions are subject to review by the authors, who take responsibility for the final version. The mathematical framework and the proofs were developed by the authors. AI assistance with code was limited to general-purpose utilities, such as retrieving layer names, counting parameters, etc. It was not used to implement the allocation algorithm or the core experimental pipeline.

\begin{reviewaddition}\subsection*{Ethics statement}
This work studies a general-purpose method for model compression. Its ethical implications depend on the underlying models, datasets, and applications. Appropriate safeguards are needed for each application.\end{reviewaddition}

\subsection*{Reproducibility statement}
Proofs are provided in Appendix~\ref{app:proofs}.
Algorithm~\ref{alg:complete_allocation} describes the allocation procedure. Appendix~\ref{app:experiments} reports the datasets, calibration and evaluation protocols, quantization settings, and runtime measurements.

\clearpage
\bibliographystyle{iclr2027_conference}
\bibliography{references}

\clearpage
\appendix
\input{appendix/proofs}
\clearpage
\makeatletter
\setlength{\@fptop}{0pt}
\makeatother
\input{appendix/experiments_appendix}
\clearpage
\input{appendix/additional_drunet}

\end{document}

%% file: sections/01_intro_related.tex
\section{Introduction}
\label{sec:intro}

Large neural networks can exceed the memory and energy budgets of embedded hardware \citep{choukroun2019low}. This is a major bottleneck for embedded vision systems, including autonomous vehicles and avionics, where models must run with limited onboard resources. On the one hand,  reducing model size and numerical precision lowers storage and memory traffic. On compatible hardware, it can also reduce computation and energy use \citep{wang2019haq}. But on the other hand, such a compression step deteriorates the quality of the networks. This is the reason why, for such applications,  there is a need to find an optimal trade-off between the accuracy of a network and its memory/energy budget.  The goal of this paper is to propose a complete post-training quantization (PTQ) framework that tackles this problem.

Post-training quantization  compresses a trained network by lowering the precision of its weights, using a small calibration set. Most PTQ pipelines use the same bit-width for every layer. This can be suboptimal because layers have different sensitivities to quantization leading to different effects on the network output. Mixed-precision PTQ addresses this by assigning a varying bit-width to each layer under a memory budget. In consequence an algorithm for mixed-precision PTQ  must decide which layers need more precision.
This flexibility can improve the quality--memory trade-off at a fixed budget \citep{dong2020hawqv2,ma2023ompq}. However, choosing these layerwise bit-width configurations jointly is difficult. For example, consider a network with 36 quantized layers and six candidate bit-widths per layer, from 3 to 8 bits. There are then $6^{36}\approx 10^{28}$ possible assignments before applying the memory constraint. Exhaustive evaluation is therefore impractical.

A practical alternative to exhaustive search must also meet two requirements. First,  the quantized model should retain its quality when calibration inputs are degraded. This matters when calibration samples are acquired from sensors affected by noise, missing measurements, or acquisition faults. Second, selection should be fast enough to repeat when the model, calibration set, candidate precisions, or memory budget changes. This cost matters during design exploration, where many settings may need to be tested.
To avoid evaluating every configuration,  mixed-precision methods estimate layer sensitivity on calibration data. Existing scores use Hessian information \citep{dong2019hawq,dong2020hawqv2,yao2021hawqv3}, block reconstruction and second-order loss information \citep{li2021brecq}, network orthogonality \citep{ma2023ompq}, pairwise cross-layer loss terms \citep{deng2023clado}, or a learned predictor \citep{mills2025qua2sedimo}. The bit-width assignment is then obtained by learned search \citep{wang2019haq}, Pareto search \citep{pandey2023practical}, linear programming \citep{ma2023ompq}, or integer programming \citep{yao2021hawqv3,deng2023clado}. The main drawback of these methods is that they can require repeated network evaluations, pairwise interactions, or an external solver, which make the allocation stage slow and computationally expensive. Their scores are also estimated from a small calibration set, whose quality and composition affect the final quantized model \citep{yuan2023benchmark,williams2024impact}.

 Theoretical analyses of quantization error can also guide the choice of sensitivity measure. \citet{gonon2024} derive a global  deterministic error bound for PTQ controlled by a worst-case norm quantity. \citet{houache2026on} provide a finer analysis by using layer dependent operator norms, yielding tighter bounds that better follow the behavior of PTQ observed in practice. This suggests that keeping the layer error information is useful for mixed-precision decisions. Note that while they offer important qualitative insight on how error is propagated, such deterministic results are pessimistic in nature, as worst-case scenarios are not likely to occur in practical settings.

\paragraph{Contributions.} 
We propose a probabilistic layerwise analysis of quantization error for a simple, yet fast and robust, mixed-precision post-training quantization algorithm. Our contributions are:
\begin{itemize}
\item  \textbf{A probabilistic layerwise analysis.}
We give probabilistic upper bounds of mixed-precision quantization error. We separate the error propagated from earlier layers from the perturbation introduced at the current layer (Section~\ref{sec:layerwise}). This keeps the source of each quantization error visible and generalize previous worst-case deterministic analysis. We also derive a global bound for modern architectures, including self-attention (Appendix~\ref{app:proof_dag}).   

\item 
\textbf{A separable, solver-free allocation algorithm.}
Using the theoretical insight from Section~\ref{sec:layerwise}, we propose a fast and robust mixed-precision PTQ algorithm. It computes separable local scores from one full-precision pass per calibration image, then assigns bit-widths with a simple greedy rule and no external solver.

\item \textbf{Fast and robust}  \textbf{bit-width allocation.}
On image denoising tasks with a DRUNet, our method combines competitive quality under clean calibration with stable PSNR under both tested calibration corruptions, at target budgets (Section~\ref{sec:robustness}).  Bit-width allocation is $28\times$ to $2{,}570\times$ faster than the studied baselines in matched timing comparisons (Section~\ref{sec:runtime}). The same score and greedy rule also transfer to latent diffusion, where our method gives competitive FID and the highest mean PSNR to FP32 in the paired comparison (Section~\ref{sec:ldm}).
\end{itemize}

\section{Related work}
\label{sec:related}

\paragraph{Error analysis for quantized networks.}
\citet{gonon2024} derive a deterministic output bound for the quantization error, controlled by a global worst-case operator norm for the network. \citet{houache2026on} instead retain layer-dependent operator norms, giving a finer account of how the error propagates through a quantized network. These deterministic bounds reveal useful structure but can be pessimistic on data. Probabilistic rounding-error analyses use moments instead of worst-case accumulation
\citep{higham2019probabilistic}. We follow this direction while keeping separate roles for our two results: the global bound evaluates a fixed bit-width configuration, and the layerwise bound exposes the local error term used for allocation. Quantization Error Propagation (QEP) corrects layer weights before quantization to reduce the effect of errors introduced by earlier layers \citep{arai2025qep}. Our method uses local error scores to choose the bit-width assigned to each block, with the quantizer kept fixed.

\paragraph{Sensitivity-driven mixed precision.}
Existing methods differ in two parts: how they measure sensitivity to quantization of layers and how they meet the budget. Sensitivity can come from Hessian statistics in HAWQ and HAWQ-V2 \citep{dong2019hawq,dong2020hawqv2}, block reconstruction in BRECQ \citep{li2021brecq}, network orthogonality in OMPQ
\citep{ma2023ompq}, quantizer-group perturbations in AIMET AMP \citep{pandey2023practical}, pairwise interactions in CLADO \citep{deng2023clado}, or topological entropy in GMPQ-TE \citep{li2025gmpqte}. The budget is then handled by a learned policy in HAQ \citep{wang2019haq}, Pareto search in AIMET AMP, or linear and integer optimization in HAWQ-V3, OMPQ, CLADO, and GMPQ-TE \citep{yao2021hawqv3,ma2023ompq,deng2023clado,li2025gmpqte}. Our method changes both parts: its score comes from the local term of our layerwise quantization error analysis, and all sensitivity scores are computed once before performing a solver-free greedy allocation. Earlier greedy PTQ work updates layer sensitivity after each decision \citep{satoki2021greedy}; our scores remain fixed and can be reused across budgets. Matched measurements in Section~\ref{sec:runtime} show bit-width allocation speed-ups from $28\times$ to $2{,}570\times$ over the studied baselines.

\paragraph{Calibration and quantization backends.}
PTQ relies on a small calibration set, whose size and composition can affect the quantized model \citep{yuan2023benchmark,williams2024impact}. QDrop addresses overfitting during reconstruction \citep{qdrop2022}. AdaRound learns weight-rounding decisions \citep{nagel2020adaround}, BRECQ reconstructs quantized blocks \citep{li2021brecq}, and PTQ4SR adapts scale
calibration to image restoration \citep{tu2023ptq4sr}. These methods change how the quantized weights or scales are obtained. Our DRUNet comparisons share a round-to-nearest (RTN) or TF-Enhanced backend, without AdaRound; AdaRound is used in the diffusion experiment. We then corrupt only the calibration inputs when testing allocation stability.

\paragraph{Quantization of diffusion models.}
Diffusion PTQ must account for the changing distributions across denoising steps. Q-Diffusion and PTQ4DM adapt calibration and reconstruction to this setting \citep{li2023qdiffusion,shang2023ptq4dm}. PTQD models
quantization noise during denoising \citep{he2023ptqd}, TFMQ-DM preserves temporal features during reconstruction \citep{huang2024tfmq}, and Qua$^2$SeDiMo learns sensitivity estimates for mixed-precision choices \citep{mills2025qua2sedimo}. Our contribution is different: we reuse the Q-Diffusion/AdaRound backend but apply to a pretrained latent diffusion model the same local score and greedy rule used for DRUNet \citep{rombach2022ldm}.

%% file: sections/02_setup_theory_method.tex
\section{Setup and notation}
\label{sec:setup}

\paragraph{Network and quantization.}
We first introduce the notation for a sequential network to keep the section simple. Appendix~\ref{app:proof_dag} gives the corresponding rules for the full architectures used in our experiments. The network has depth $L\ge1$ and linear layers without biases. Layer $\ell$ has weights $W_\ell\in\R^{N_\ell\times N_{\ell-1}}$. At bit-width $b_\ell\ge2$, its quantized weights are $W'_\ell:=\Quant{b_\ell}(W_\ell)$, where $\Quant{b}$ is the chosen quantization method, e.g. round-to-nearest. The input $x$ follows a distribution $p$ on $\R^{N_0}$. The full-precision and quantized forward passes are $S_\ell:=W_\ell S_{\ell-1}$ and $S'_\ell:=W'_\ell S'_{\ell-1}$, with $S_0=S'_0=x$. Their difference, i.e. the quantization error, at depth $\ell$ is $\Delta_\ell:=S_\ell-S'_\ell$. It satisfies the exact decomposition used in deterministic analyses \citep{gonon2024,houache2026on}:
\begin{equation} 
	\label{eq:exact_decomposition}
	\Delta_\ell \;=\; W_\ell \Delta_{\ell-1} \;+\; (W_\ell - W'_\ell)\, S'_{\ell-1}.
\end{equation}
The first term in Equation~\eqref{eq:exact_decomposition} propagates the error from previous layers. The second is the \emph{local quantization error} introduced at layer $\ell$. These two terms drive the analysis below.

\begin{definition}[Random amplification ratios]
	\label{def:amplification_ratios}
	For each layer $\ell \ge 1$, define the random variables
	\begin{equation}
		\label{eq:def_XY}
		X_\ell :=\frac{\norm{W_\ell \Delta_{\ell-1}}} {\norm{\Delta_{\ell-1}}}, \qquad Y_\ell := \frac{\norm{(W_\ell -W'_\ell) S'_{\ell-1}}} {\norm{S'_{\ell-1}}}.
	\end{equation}
	We set $X_\ell = 0$ when $\norm{\Delta_{\ell-1}} = 0$ and $Y_\ell = 0$ when $\norm{S'_{\ell-1}} = 0$.
\end{definition}

These ratios measure how much the propagated and local terms are amplified on the current input. They satisfy
\[
0 \le X_\ell \le \lVert W_\ell\rVert_{\mathrm{op}}, \qquad 0 \le Y_\ell \le \lVert W_\ell - W'_\ell\rVert_{\mathrm{op}}.
\]
The operator norms cover the worst case error while these ratios track the error actually reached by the data, making them suitable for a probabilistic approach. Our analysis below uses their moments.

\begin{definition}[$\delta$-probabilistic upper bound]
	\label{def:probabilistic_upper_bound}
	Let $Z$ be a nonnegative random variable and let $\delta \in [0,1)$. A deterministic constant $A$ is a \emph{$\delta$-probabilistic upper bound} for $Z$ if $\Pp(Z \le A) \ge 1 - \delta$.
\end{definition}

For each layer $\ell$, the recursion in Section~\ref{sec:layerwise} uses a $\tilde{\delta}_\ell$-probabilistic upper bound $B_\ell$ for $\norm{S'_\ell}$. At the input, we assume a bounded domain and choose $B_0$ such that $\Pp\bigl(\norm{x} \le B_0\bigr) = 1$, $\tilde{\delta}_0 = 0$.

\paragraph{Mixed-precision allocation problem}
Each layer $\ell\in\{1,\ldots,L\}$ uses a bit-width $b_\ell$ from a finite candidate set $\mathcal B_\ell$, and $n_\ell$ is its number of weights. A mixed-precision bit-width configuration $\boldsymbol{b} =(b_\ell)_{\ell=1}^L$ has cost $\mathrm{Cost}(\boldsymbol{b}) = \sum_{\ell=1}^L b_\ell n_\ell$, and the admissible set under a memory budget $\mathrm{Bud}$ is $\Feas:= \{ \boldsymbol{b} : \mathrm{Cost}(\boldsymbol{b}) \le \mathrm{Bud},\; b_\ell \in \mathcal B_\ell \}$. For comparison, we express the memory budget as $\mathrm{Bud}=B\sum_{\ell=1}^L n_\ell$, where $B$ is the target average number of bits per weight. For example, $B=4$ gives the same weight-memory budget as uniform 4-bit PTQ. Under configuration $\boldsymbol{b}$, the output error is $\norm{\Delta_L(x; \boldsymbol{b})}$. The objective of \emph{robust} mixed-precision PTQ  is to solve the following minimization problem:
\begin{equation} 
	\label{eq:ideal_objective}
	\min_{\boldsymbol{b} \in \Feas} \; Q_\tau(\boldsymbol{b}), \quad \text{where} \quad
Q_\tau(\boldsymbol{b}) := \inf\{ t : \Pp(\norm{\Delta_L(x,\boldsymbol{b})} \le t) \ge \tau \}.
\end{equation}
Here $\tau$ is close to one and $\delta=1-\tau$ is the corresponding tail level. Using a quantile instead of worst-case errors permits increased robustness to outliers in available datasets. In practice, we only have access to a calibration set $\Dcal=\{x_1,\dots,x_n\}$.
This ideal solution depends on the complete bit-width configuration. Testing one configuration requires a full-network evaluation, and the number of configurations grows exponentially with the number of layers. We begin by providing a probabilistic PTQ error analysis. Section~\ref{sec:global} first, bounds the output error of a fixed configuration. Section~\ref{sec:layerwise} then keeps track of where the error is introduced. Then, Section~\ref{sec:allocation} turns the local term of the layerwise analysis into a practical allocation score.

\section{Probabilistic error bounds for mixed-precision PTQ }
\label{sec:theory}

We give two bounds. We present first a simple global result  as a warm-up;   answering the natural question: how large is the final output error? It also shows why that view alone cannot help with layerwise bit-width allocation. The following layerwise result is the main link to the allocation algorithm. Our layerwise bound keeps the information needed for mixed precision: where each error is introduced and how it is carried forward.

\subsection{A global bound for a fixed configuration}
\label{sec:global}

For a fixed configuration, define the sequential model
\begin{equation}
  \label{eq:global_recursion}
  G_0:=0, \qquad G_\ell:=X_\ell G_{\ell-1} +\norm{(W_\ell-W'_\ell)S'_{\ell-1}}.
\end{equation}
Equation~\eqref{eq:exact_decomposition} and induction give
$\norm{\Delta_\ell}\le G_\ell$ for every input. Since $x\sim p$, $G_L$ varies with the input, and one concentration step at the output gives the following result.

\begin{lemma}[Global bound] 
	\label{lem:global}
	Assume $G_L$ has finite second moment under $p$, and let $\muG := \E[G_L]$, $\sigG^2 := \Var(G_L)$. For any $\delta \in (0,1)$,
	\begin{equation} 
		\label{eq:cantelli}
		\Pp\Bigl(\norm{\Delta_L} > \muG + \sigG \sqrt{\tfrac{1-\delta}{\delta}} \Bigr) \le \delta.
	\end{equation}
\end{lemma}

Once the pointwise recursion has been established, the probabilistic step is a direct application of the one-sided Chebyshev inequality \citep{cantelli1928} (Appendix~\ref{app:proof_global}). We write the recursion for a sequential model only to keep this warm-up simple. Appendix~\ref{app:proof_dag} proves propagation rules for modern computation graphs, so the same global bound applies to architectures built from these operations, including U-Nets and vision transformers.
For a chosen configuration $\boldsymbol{b} $, the moments of $G_L$ can be estimated on the calibration set. The theorem uses the true mean and variance. Estimating them from calibration data does not provide the same guarantee automatically. The global bound controls a complete configuration, but it does not say which layer should use a higher bit-width. The next Section gives such information through probabilistic layerwise bounds.

\subsection{A layerwise bound for mixed-precision quantization}
\label{sec:layerwise}
First remark that Equation~\eqref{eq:exact_decomposition} and sub-multiplicativity give the deterministic one-step bound
\begin{equation}
\label{eq:deterministic_one_step}
\norm{\Delta_\ell} \le \normop{W_\ell}\norm{\Delta_{\ell-1}} +\normop{W_\ell-W'_\ell}\norm{S'_{\ell-1}}.
\end{equation}
This is the layer-dependent structure used by \citet{houache2026on}. The random ratios $X_\ell$ and $Y_\ell$ introduced earlier are the realized counterparts of the two operator norms in Equation~\eqref{eq:deterministic_one_step}. We use their moments to derive a probabilistic layerwise bound that keeps the source of each error visible. To do this, we apply a concentration inequality at each layer and combine them through a union bound. Thanks to this approach, we preserve the two terms of Equation~\eqref{eq:exact_decomposition}. This is the information needed by a mixed-precision allocation algorithm.
The recursion uses two bounds at each layer. The bound $B_\ell$ is a $\tilde{\delta}_\ell$-probabilistic upper bound of output $S_\ell'$:
\begin{equation}
	\label{eq:activation_threshold}
	\Pp\bigl(\norm{S'_\ell}\le B_\ell\bigr) \ge 1-\tilde{\delta}_\ell.
\end{equation}
The error bound $A_\ell$ is constructed recursively as follows. At layer $\ell$, define the event
\begin{equation}
	\label{eq:def_event} E_\ell:= \{\norm{\Delta_\ell}\le A_\ell\} \cap \{\norm{S'_\ell}\le B_\ell\}.
\end{equation}
We start from $A_0:=0$ and $E_0:=\{\norm{x}\le B_0\}$.
Conditionally on $E_{\ell-1}$, define the moments
\begin{equation}
\label{eq:conditional_moments}
\begin{aligned}
\mu_{X,\ell}&:=\E[X_\ell\mid E_{\ell-1}], &
\mu_{Y,\ell}&:=\E[Y_\ell\mid E_{\ell-1}], &
\mathrm{Cov}_\ell&:=\mathrm{Cov}(X_\ell,Y_\ell\mid E_{\ell-1}),\\
\sigma_{X,\ell}^2&:=\Var(X_\ell\mid E_{\ell-1}), &
\sigma_{Y,\ell}^2&:=\Var(Y_\ell\mid E_{\ell-1}).
\end{aligned}
\end{equation}
We assume these moments are finite.
On $E_{\ell-1}$, the decomposition in
Equation~\eqref{eq:exact_decomposition} gives
\begin{equation}
	\label{eq:layerwise_one_step}
	\norm{\Delta_\ell} \le X_\ell A_{\ell-1} + Y_\ell B_{\ell-1}.
\end{equation}
The quantities $m_\ell$ and $v_\ell$ in the theorem below are the conditional mean and variance of the right side of Equation~\eqref{eq:layerwise_one_step}.

\begin{theorem}[Layerwise bound]
	\label{thm:layerwise}
	Fix a bit-width configuration, levels $\delta_1,\dots,\delta_L\in(0,1)$, and $\tilde{\delta}_1,\dots,\tilde{\delta}_{L-1}\in[0,1)$, with
	\begin{equation}
		\label{eq:layerwise_condition}
		\varepsilon:=\delta_L+\sum_{\ell=1}^{L-1}(\delta_\ell+\tilde{\delta}_\ell)<1
	\end{equation}
	and define recursively
	\begin{equation}
\label{eq:def_A}
\begin{aligned}
m_\ell&:=\mu_{X,\ell}A_{\ell-1}+\mu_{Y,\ell}B_{\ell-1},
\qquad A_\ell:=m_\ell+\sqrt{v_\ell\,\frac{1-\delta_\ell}{\delta_\ell}},\\
v_\ell&:=A_{\ell-1}^2\sigma_{X,\ell}^2+B_{\ell-1}^2\sigma_{Y,\ell}^2+2A_{\ell-1}B_{\ell-1}\mathrm{Cov}_\ell.
\end{aligned}
\end{equation}
	Then
	\begin{equation}
		\label{eq:layerwise_bound}
		\Pp\bigl(\norm{\Delta_L}>A_L\bigr) \le \varepsilon.
	\end{equation}
In other words, the quantity $A_L$ is an $\varepsilon$-probabilistic upper bound of $\norm{\Delta_L}$.
\end{theorem}

The proof in Appendix~\ref{app:proof_layerwise} applies a conditional one-sided Chebyshev inequality (Cantelli's inequality) at each layer to $X_\ell A_{\ell-1}+Y_\ell B_{\ell-1}$. The bounds $B_\ell$ replace the norms of $S'_\ell$ by fixed values on the  considered events. The conditional mean keeps the propagated and local terms separate, while the variance also contains their covariance. This is the structure needed by the algorithm: each layer adds a local perturbation, and later layers carry earlier error forward. However, computing the exact propagated contribution for every joint bit-width configuration would remain expensive. Section~\ref{sec:score} therefore keeps the local quantity that identifies the source layer and evaluates it on full-precision activations. This result is also of mathematical interest because it generalizes the deterministic layerwise analysis of \citet{houache2026on} to a probabilistic setting.

\section{An algorithm for fast and robust mixed-precision PTQ}
\label{sec:allocation}
\label{sec:score}

To design our algorithm, we want to estimate the local influence of each layer and candidate bit-width on quantization error. Computing the exact influence of a layer requires a full pass in the quantized network for a given configuration $\boldsymbol{b}$. We therefore retain the local quantization error identified by Theorem~\ref{thm:layerwise} and evaluate it before choosing a complete configuration. By definition, the local perturbation  $Y_\ell$ introduced at layer $\ell$ satisfies
\begin{equation}
\label{eq:local_perturbation_identity}
Y_\ell\norm{S'_{\ell-1}} =\norm{(W_\ell-W'_\ell)S'_{\ell-1}}.
\end{equation}

\begin{algorithm}[t]
\caption{Local score construction and mixed-precision allocation}
\label{alg:complete_allocation}
\KwIn{Full-precision network; calibration set $\Dcal$; blocks and sizes $(g,n_g)$; candidate sets $(\mathcal B_g)_{g\in\Blocks}$; feasible budget $\mathrm{Bud}$; $\delta_{\mathrm{score}}$.}
\KwOut{Bit-width configuration $(b_g)_{g\in\Blocks}$.}
Initialize running statistics for every pair $(g,b)$\;
\ForEach{$x_i\in\Dcal$}{
  Run one full-precision forward pass and capture the required layer inputs\;
  \ForEach{$g\in\Blocks$ and $b\in\mathcal B_g$}{
    Compute $P_{g,b}(x_i)$ from Equation~\eqref{eq:block_pert} and update its running statistics\; }  Release the captured activations\;}
\lForEach{$g\in\Blocks$ and $b\in\mathcal B_g$}{Compute $\sloc_g(b)=\hat\mu_{g,b}+\hat\sigma_{g,b}/\sqrt{\delta_{\mathrm{score}}}$}
Set $b_g\leftarrow\max\mathcal B_g$ for every block\;
\While{$\mathrm{Cost}(\boldsymbol{b})>\mathrm{Bud}$}{
\ForEach{$g\in\Blocks$ with $b_g>\min\mathcal B_g$}{
  Set $b_g^-$ to the next lower candidate and compute $\Delta_g(b_g^{-})$ (Equation~\eqref{eq:greedy_ratio}) \;}
  Select a block $g^\star$ minimizing $\Delta_g(b_g^{-})$, breaking ties deterministically \;
 Update  $b_{g^\star}\leftarrow b_{g^\star}^-$ \;
}
\Return{$(b_g)_{g\in\Blocks}$}\;
\end{algorithm}

The moments of $Y_\ell$ are then used to build the layerwise probabilistic error bounds. The main idea of our algorithm is to consider this local perturbation to design a sensitivity score.
The analysis above is layerwise. For efficiency, our practical algorithm groups layers into allocation blocks $g\in\Blocks$, as in other grouped mixed-precision methods \citep{pandey2023practical,sun2022entropy}. This reduces the number of decisions and matches the block partitions used in our comparisons. All layers $\ell\in g$ in a given block share one bit-width $b_g\in\mathcal B_g$. Let $n_g:=\sum_{\ell\in g}n_\ell$ be the number of weights in block $g$; the practical cost is $\mathrm{Cost}(\boldsymbol{b})=\sum_{g\in\Blocks}b_g n_g$. For a block $g$ quantized at bit-width $b$, we define the approximated local quantization error
\begin{equation} 
	\label{eq:block_pert}
	P_{g,b}(x):= \sum_{\ell\in g} \bigl\lVert \bigl(W_\ell-\Quant{b}(W_\ell)\bigr)S_{\ell-1}(x) \bigr\rVert_2.
\end{equation}
The sum gives one value for each image in the calibration set and each block $g$.  In practice, we evaluate the local perturbation on the full-precision layer input $S_{\ell-1}(x)$ instead of $S_{\ell-1}'(x)$ as in \eqref{eq:local_perturbation_identity}. This is justified in our case because our algorithm  navigates configurations yielding low quantization errors, i.e $S_{\ell-1}(x) \approx S'_{\ell-1}(x) $ on the whole optimization path of Algorithm~\ref{alg:complete_allocation}. \revadd{The algorithm starts from the highest candidate bit-widths, where quantization errors are expected to be small.} Using the same full-precision inputs makes the values  $P_{g,b}$  independent of the bit-widths assigned to other blocks. 
We summarize the values of $P_{g,b}$ into one score on the calibration set by their mean and standard deviation:
\begin{equation} 
	\label{eq:local_score}
\sloc_g(b) := \hat\mu_{g,b} + \frac{\hat\sigma_{g,b}}{\sqrt{\delta_{\mathrm{score}}}},
\end{equation}
where $\hat\mu_{g,b}$ and $\hat\sigma_{g,b}$ are the sample mean and sample standard deviation of $P_{g,b}(x_i)$ on $\Dcal$. The mean captures the average perturbation, while the standard deviation penalizes variation across calibration images. The value of  $\delta_{\mathrm{score}}$ only sets the fixed weight of the dispersion term; it is not a confidence level for the selected allocation.
To calculate $\sloc_g(b)$, each image in the calibration dataset is passed once through the full-precision network in order  to update the scores for all candidate bit-widths. Computing $\sloc_g(b)$ for every block and candidate bit-width gives the scores used by our algorithm described below. Each score depends only on one pair $(g,b)$ and can be reused for any configuration  $\boldsymbol{b}$.

\paragraph{Bit-width allocation with a greedy algorithm}
\label{sec:greedy}

Our algorithm is initialized with every block at its largest candidate bit-width. The idea is to greedily decrease the bit-width block by block, based on the sensitivity score  $\sloc_g$. Given a current value $b_g$, let $b_g^{-} < b_g$ be the next lower candidate bit-width. At each step, the algorithm selects the block with the smallest score increase per weight bit saved,
\begin{equation}
\label{eq:greedy_ratio}
\Delta_g(b_g^{-}):=\frac{\bigl[\sloc_g(b_g^{-})-\sloc_g(b_g)\bigr]_+}
     {(b_g-b_g^{-})\,n_g},
\end{equation}
and sets $b_g\leftarrow b_g^{-}$. The denominator is the number of weight bits saved by this change, and $[.]_+=\max(.,0)$ sets a negative score difference to zero. The process stops when the budget is met. All decisions use the precomputed scores, so this stage requires no new network evaluation and no external solver. Algorithm~\ref{alg:complete_allocation} precisely describes our method.

%% file: sections/03_experiments_revised.tex
\section{Experiments}
\label{sec:experiments}

We evaluate two applications: image denoising with DRUNet and image generation with a latent diffusion U-Net. We first test whether our algorithm gives competitive allocations under clean calibration. We then keep the test set fixed and corrupt only the calibration inputs. At moderate budgets, our method matches the strongest references under clean calibration and performs better under  corrupted calibration dataset. Finally, we measure the time needed to compute the scores and assign bit-widths, and transfer the method to a latent diffusion U-Net. Some PTQ methods also optimize rounding decisions or quantization ranges. Our contribution is a local score and a simple greedy allocation rule; each DRUNet comparison uses a shared quantizer.

\subsection{Experimental setup}
\label{sec:common_setup}

\paragraph{DRUNet denoising.}
We use the pretrained color DRUNet from \citep{zhang2021plug} on $256\times256$ CelebA images \citep{liu2015celeba}. The input noise has standard deviation $\sigma_{\mathrm{in}}=0.1$. DRUNet has $32.64$ million weights in $64$ convolutions, grouped into $36$ allocation blocks. We use three calibration sets of $256$ images and the same $100$ clean test images. Quantization is weight-only; intermediate layer outputs remain in FP32. PSNR is averaged over images and calibration runs.

\paragraph{Matched comparisons.}
The reference methods use different quantizers and candidate bit-widths, so we match our algorithm to each reference separately when possible. Within each matched \revadd{baseline} comparison, both methods use the same quantization backend and allocation blocks. OMPQ-ORM scores blocks by network orthogonality and uses an algorithm for mixed-integer linear programming. CLADO models pairwise interactions between layer choices and uses  integer quadratic programming. AIMET measures quantizer-group sensitivity and searches a Pareto set. Table~\ref{tab:external_protocols} in the appendix summarizes these settings and an additional output signal to quantization noise ratio (SQNR) score control. We compare the bit-width configurations produced by the methods and the resulting model quality. These are allocation comparisons, not reproductions of each paper's complete PTQ pipeline.
For OMPQ-ORM, both algorithms use symmetric per-channel round-to-nearest quantization (RTN) with candidates from $3$ to $8$ bits. Each output channel has its own scale, and each weight is rounded to the nearest representable value. CLADO uses the same RTN backend with candidates $\{3,4,8\}$, while our main configuration uses $\{3,4,5,6,7,8\}$. For AIMET, both algorithms use TF-Enhanced with candidates $\{4,6,8\}$. This scheme selects quantization ranges to reduce quantization error. The adapted output-SQNR control quantizes one block at a time and measures the SQNR at the final DRUNet output, then uses our greedy rule. AIMET AMP is evaluated separately. Realized precision is the average bit-width weighted by the number of weights. It measures weight storage, not inference time.

\paragraph{Corrupted calibration datasets.}
Only the inputs used to build the allocation are corrupted. We use $\widetilde x=M\odot x+\epsilon$, where $M$ is a binary corruption pattern and $\epsilon$ is the same denoising noise as in the clean protocol. Each entry of $M$ is kept with probability $1-p$ and set to zero with probability $p$. We test one spatial pattern shared across RGB channels at $p=0.5$, and independent per-channel patterns at $p=0.9$. Calibration images and noise are paired across algorithms, and the test set is unchanged. These are distinct stress conditions.
We report $95\%$ confidence intervals for mean PSNR differences, comparing the methods on the same test images (Appendix~\ref{app:drunet_details}). Our score uses $\delta_{\mathrm{score}}=0.05$ in Equation~\eqref{eq:local_score}. Visual examples are in Appendix~\ref{app:qualitative}. In the top row of Figure~\ref{fig:clado_qualitative}, CLADO leaves more visible residual grain than our method under shared-channel calibration corruption.
\subsection{Experiments with clean calibration dataset}
\label{sec:clean_drunet}

Table~\ref{tab:allocation_benchmark} compares the different methods at a target average bit-width of $4$. Uniform quantization $B = 4$ loses $5.687$~dB relative to FP32, showing that the importance of mixed-precision PTQ. Our method improves OMPQ-ORM by $1.814$~dB and matches the performance of CLADO (only $0.008$~dB below). For our method minus CLADO, the $95\%$ confidence interval for the mean difference is $[-0.015,-0.001]$~dB.
\begin{table}[t]
\caption{Clean DRUNet results with the RTN backend at target budget $B=4$, equivalent to uniform 4-bit quantization. Data-dependent values are averaged over three calibration sets and the same $100$ test images.}
\label{tab:allocation_benchmark}
\centering
\setlength{\aboverulesep}{1pt}
\setlength{\belowrulesep}{1pt}
\normalsize
\begin{tabular}{lccc}
\toprule
Method & Candidate bit-widths & Mean bit-width & PSNR (dB) \\
\midrule
FP32 & 32 & 32.000 & 34.664 \\
Uniform, $B=3$ & 3 & 3.000 & 24.497 \\
Uniform, $B=4$ & 4 & 4.000 & 28.978 \\
OMPQ-ORM, $B=4$ & 3--8 & 3.988 & 32.587 \\
CLADO, $B=4$ & 3, 4, 8 & 3.988 & 34.409 \\
Our method, $B=4$ & 3--8 & 3.987 & 34.401 \\
\bottomrule
\end{tabular}
\end{table}
The complete experiments with different budgets give a complete picture (Appendix~\ref{app:clean_sweeps}). Our method outperforms OMPQ-ORM and the adapted output-SQNR score at every reported budget. CLADO is slightly better at the smallest target budget, $B=3.5$, while their absolute PSNR gap remains below $0.01$~dB from $B=4$ to $B=5$ in their main candidate spaces. Restricting our method to the same W3/W4/W8 candidates leads to the same conclusion (Appendix~\ref{app:candidate_matched}). \revadd{The speed of our method allows us to explore a larger allocation space, with one bit-width decision per layer instead of per block. At $B=3.5$, this matches CLADO on CelebA and slightly exceeds it on BSDS500 under clean calibration (Appendix~\ref{app:allocation_granularity}).} However, the next sections show that our algorithm shows  large improvements in both computational efficiency and robustness. 

Table~\ref{tab:aimet_clean} tests a separate TF-Enhanced backend. At $B=5$, our method matches AIMET  ($0.010$~dB higher) while using a slightly lower average bit-width; the $95\%$ confidence interval for the mean difference is $[0.007,0.013]$~dB. At $B=7$, the difference is negligible ($-0.0003$~dB, with interval $[-0.0011,0.0005]$~dB). At $B=6$, our PSNR is higher, but the realized precisions differ.
\begin{table}[t]
\caption{Clean DRUNet comparison with the common AIMET TF-Enhanced backend.}
\label{tab:aimet_clean}
\centering
\setlength{\aboverulesep}{1pt}
\setlength{\belowrulesep}{1pt}
\setlength{\tabcolsep}{3.5pt}
\normalsize
\begin{tabular}{cccccc}
\toprule
$B$ & Our mean bit-width & AIMET mean bit-width & Our PSNR & AIMET PSNR & Difference \\
\midrule
5 & 4.932 & 4.961 & 34.639 & 34.629 & $+0.010$ \\
6 & 5.984 & 5.716 & 34.653 & 34.636 & $+0.017$ \\
7 & 6.811 & 6.835 & 34.655 & 34.656 & $-0.0003$ \\
\bottomrule
\end{tabular}
\end{table}
Overall, the proposed score reaches the quality of the stronger references for quantization down to an average $4$ bits per weight, without using pairwise optimization or Pareto search. Bit-width allocation times are reported in Section~\ref{sec:runtime}.

\subsection{Stability under calibration corruption}
\label{sec:robustness}

Table~\ref{tab:robustness_matrix} compares the methods at one representative budget while keeping the clean test set unchanged. With the main W3--W8 candidates, our PSNR remains within $0.03$~dB of values obtained with a clean calibration dataset from $B=4$ to $B=5$. OMPQ-ORM is also stable but reaches lower absolute quality, and AIMET remains close to our method at $B=5$. In contrast, CLADO matches our clean result at $B=4$ but loses more than $5$~dB under either corruptions. These results are specific to the protocols in Table~\ref{tab:external_protocols}.
\begin{table}[t]
\caption{DRUNet results when only the calibration inputs used for allocation are corrupted. Each pair of rows uses the same quantization backend. PSNR is in dB. CLADO uses W3/W4/W8 and our main method uses W3--W8; realized average bit-widths are therefore shown explicitly.}
\label{tab:robustness_matrix}
\centering
\setlength{\aboverulesep}{1pt}
\setlength{\belowrulesep}{1pt}
\setlength{\tabcolsep}{2.5pt}
\normalsize
\begin{tabular}{llcccccc}
\toprule
\multirow{2}{*}{Comparison} & \multirow{2}{*}{Method}
& \multicolumn{2}{c}{clean}
& \multicolumn{2}{c}{$p=0.5$ shared}
& \multicolumn{2}{c}{$p=0.9$ per-channel} \\
\cmidrule(lr){3-4}\cmidrule(lr){5-6}\cmidrule(lr){7-8}
& & \shortstack[c]{Mean\\bit-width} & PSNR & \shortstack[c]{Mean\\bit-width} & PSNR & \shortstack[c]{Mean\\bit-width} & PSNR \\
\midrule
\multirow{2}{*}{AIMET ($B=5$)}
& Our method & 4.932 & 34.639 & 4.940 & 34.640 & 4.980 & 34.637 \\
& AIMET & 4.961 & 34.629 & 4.970 & 34.529 & 4.792 & 34.512 \\
\midrule
\multirow{2}{*}{OMPQ-ORM ($B=4$)}
& Our method & 3.987 & 34.401 & 3.991 & 34.429 & 3.925 & 34.429 \\
& OMPQ-ORM & 3.988 & 32.587 & 3.988 & 32.587 & 3.988 & 32.576 \\
\midrule
\multirow{2}{*}{CLADO ($B=4$)}
& Our method & 3.987 & 34.401 & 3.991 & 34.429 & 3.925 & 34.429 \\
& CLADO & 3.988 & 34.409 & 3.939 & 27.824 & 3.987 & 28.928 \\
\bottomrule
\end{tabular}
\end{table}
The same pattern holds across the complete CLADO experiments with different budgets, our largest observed gain is $7.509$~dB at $B=5$ under the shared corruption (Figure~\ref{fig:clado_corruption_curves}). 
The allocation diagnostics support this result. At $B=4$, the Spearman correlations between clean and corrupted
bit-width configurations are $0.961$ and $0.987$ for our method, compared with $0.060$ and $0.313$ for CLADO, under shared and per-channel corruption, respectively (Figure~\ref{fig:clado_stability}). Our method therefore better preserves the relative ordering of blocks by assigned bit-width. The W3/W4/W8 candidate-matched control leads to the same conclusion: our method remains more than $5$~dB ahead under both corruptions. This gain is therefore not only due to the additional intermediate bit-widths, but from our specific algorithm relying on local probabilistic sensitivity scores. Full curves and allocation diagnostics are in Appendix~\ref{app:full_curves} and random qualitative figures are in Appendix~\ref{app:qualitative}. \revadd{On BSDS500, our main W3--W8 method also gives comparable clean quality to CLADO and higher PSNR under both corruptions (Appendix~\ref{app:bsds_results}).}

\subsection{Measured cost}
\label{sec:runtime}

On DRUNet with $256$ clean calibration images, our method selects the bit-widths $28\times$ faster than AIMET AMP and $30\times$ faster than OMPQ-ORM with a fixed $\beta$. When OMPQ must also search for $\beta$, our method is $52\times$ faster. With the same W3/W4/W8 candidates, our method is $2{,}570\times$ faster than CLADO. \revadd{Times cover score computation and allocation, with medians over three runs.} CLADO uses two A100 GPUs, whereas our method uses one in this comparison (Appendix~\ref{app:runtime}). The same scores can be reused when only the memory budget changes. \revadd{The estimated layerwise selection time is $24.32$ seconds on one A100. For comparison, CLADO took about $900\times$ longer in the separate benchmark reported in Table~\ref{tab:recorded_runtime}.}

\subsection{Latent-diffusion transfer}
\label{sec:ldm}
We apply the same local score and greedy rule to DRUNet denoising and latent-diffusion generation, without adding a diffusion-specific allocation score or solver. Q-Diffusion and TFMQ-DM use a uniform 4-bit weight allocation in the PSNR comparison. Q-Diffusion uses the same Q-Diffusion/AdaRound backend \citep{li2023qdiffusion}. TFMQ-DM is a separate diffusion-specific pipeline with temporal feature-aware reconstruction \citep{huang2024tfmq}. Using $50{,}000$ generated images per method, our mixed-precision model obtains a FID of 20.853, compared with 22.626 for Q-Diffusion. \revadd{Our quantized model is competitive for generation.} \revadd{Across $50{,}000$ outputs paired by generation seed,} our method has the highest mean PSNR to FP32: 34.644~dB, compared with 24.545~dB for the uniform-W4 TFMQ-DM configuration. \revadd{Appendix~\ref{app:ldm_transfer} gives full results, PSNR distributions, and visual comparisons showing fewer artifacts with our quantized model.}

%% file: appendix/proofs.tex

\section{Proofs}
\label{app:proofs}

With, $x \sim p$, the ratios $X_\ell, Y_\ell$ are defined
in~\eqref{eq:def_XY} with their zero conventions, and all conditional moments of Theorem~\ref{thm:layerwise} assumed finite. Recall the bounded domain assumption of Section~\ref{sec:setup}: $B_0$ is deterministic with $\Pp(\norm{x} \le B_0) = 1$.

\subsection{Proof of Theorem~\ref{thm:layerwise} (layerwise bound)}
\label{app:proof_layerwise}

For $\ell \ge 1$, define the comparison variable
\begin{equation}
	\label{eq:def_tildeZ}
	\tilde Z_\ell := X_\ell A_{\ell-1} + Y_\ell B_{\ell-1}.
\end{equation}
This replaces the two random norms by the deterministic bounds $A_{\ell-1}$ and $B_{\ell-1}$.

\paragraph{Step 1: domination on the conditioning event.}
Taking norms in the exact decomposition~\eqref{eq:exact_decomposition} and using the triangle inequality gives
\begin{equation}
	\label{eq:Delta_le_Z}
	\begin{aligned}
		\norm{\Delta_\ell} &\le \norm{W_\ell\Delta_{\ell-1}} +
		\norm{(W_\ell-W'_\ell)S'_{\ell-1}} \\ &= X_\ell\norm{\Delta_{\ell-1}} + Y_\ell\norm{S'_{\ell-1}}.
	\end{aligned}
\end{equation}
The equality remains valid when either denominator in the definitions of $X_\ell$ and $Y_\ell$ is zero, by the corresponding zero convention. Therefore, on
\begin{equation}
	E_{\ell-1} = \{\norm{\Delta_{\ell-1}}\le A_{\ell-1}\} \cap \{\norm{S'_{\ell-1}}\le B_{\ell-1}\},
\end{equation}
we have

\begin{equation}
	\label{eq:Z_le_tildeZ}
	\norm{\Delta_\ell} \le X_\ell A_{\ell-1}+Y_\ell B_{\ell-1} = \tilde Z_\ell .
\end{equation}

\paragraph{Step 2: exact conditional moments and conditional
	one-sided Chebyshev.}
Suppose that $\Pp(E_{\ell-1})>0$. Since $A_{\ell-1}$ and
$B_{\ell-1}$ are deterministic constants, the first two conditional moments of $\tilde Z_\ell$ are
\begin{equation}
	\label{eq:exact_moments}
	\E[\tilde Z_\ell\mid E_{\ell-1}]=m_\ell, \qquad
	\Var(\tilde Z_\ell\mid E_{\ell-1})=v_\ell
\end{equation}

where $m_\ell$ and $v_\ell$ are defined in~\eqref{eq:def_A}.
If $v_\ell>0$, set
\begin{equation}
	t_\ell := \sqrt{v_\ell\frac{1-\delta_\ell}{\delta_\ell}} >0.
\end{equation}
Since $A_\ell=m_\ell+t_\ell$, the one-sided Chebyshev inequality
(Cantelli's inequality) gives
\begin{equation}
	\label{eq:cond_cheb}
	\begin{aligned}
		\Pp(\tilde Z_\ell>A_\ell\mid E_{\ell-1}) &\le \Pp(\tilde Z_\ell-m_\ell\ge t_\ell\mid E_{\ell-1}) \\
		&\le \frac{v_\ell}{v_\ell+t_\ell^2} = \delta_\ell
	\end{aligned}
\end{equation}

If $v_\ell=0$, then $\tilde Z_\ell=m_\ell=A_\ell$ almost surely
conditionally on $E_{\ell-1}$. The probability of the strict event in
\eqref{eq:cond_cheb} is therefore zero. Thus~\eqref{eq:cond_cheb} holds in both cases.

\paragraph{Step 3: the first layer.}

Let
\begin{equation}
	F_1:=\{\norm{\Delta_1}>A_1\}.
\end{equation}
The event $E_0$ holds almost surely. Hence the conditional probability in~\eqref{eq:cond_cheb} is well defined at $\ell=1$. Moreover, \eqref{eq:Z_le_tildeZ} gives
\begin{equation}
	F_1 \subseteq \bigl(\{\tilde Z_1>A_1\}\cap E_0\bigr) \cup E_0^c.
\end{equation}

It follows that
\begin{equation}
	\begin{aligned}
		\Pp(F_1) &\le \Pp(\tilde Z_1>A_1\mid E_0)\Pp(E_0) +\Pp(E_0^c) \\ &\le \delta_1.
	\end{aligned}
\end{equation}

If $L=1$, then $\varepsilon=\delta_1$, and this already proves the
theorem. We now assume $L\ge2$.

\paragraph{Step 4: induction up to layer $L-1$.}
For $1\le\ell\le L-1$, we prove

\begin{equation}
	\label{eq:IH}
	\Pp(E_\ell^c) \le \sum_{k=1}^{\ell}(\delta_k+\tilde{\delta}_k).
\end{equation}

\emph{Base case $\ell=1$.}
Since
\begin{equation}
	E_1^c = F_1\cup\{\norm{S'_1}>B_1\},
\end{equation}

Equation~\eqref{eq:activation_threshold} and the preceding estimate give
\begin{equation}
	\Pp(E_1^c) \le \delta_1+\tilde{\delta}_1.
\end{equation}
This proves~\eqref{eq:IH} at $\ell=1$.

\emph{Inductive step.}
Let $2\le\ell\le L-1$ and assume~\eqref{eq:IH} at $\ell-1$. Then
\begin{equation}
	\Pp(E_{\ell-1}) \ge 1-\sum_{k=1}^{\ell-1}(\delta_k+\tilde{\delta}_k) \ge 1-\varepsilon > 0
\end{equation}
Thus the conditional bound~\eqref{eq:cond_cheb} is well defined at
layer $\ell$.

Let
\begin{equation}
	F_\ell:=\{\norm{\Delta_\ell}>A_\ell\}.
\end{equation}

By~\eqref{eq:Z_le_tildeZ},
\begin{equation}
	F_\ell \subseteq \bigl(\{\tilde Z_\ell>A_\ell\}\cap E_{\ell-1}\bigr) \cup E_{\ell-1}^c.
\end{equation}

Therefore, using~\eqref{eq:cond_cheb} and the induction hypothesis,
\begin{equation}
	\label{eq:step_bound}
	\begin{aligned}
		\Pp(F_\ell) &\le \Pp(\tilde Z_\ell>A_\ell\mid E_{\ell-1}) \Pp(E_{\ell-1}) + \Pp(E_{\ell-1}^c) \\
		&\le \delta_\ell + \sum_{k=1}^{\ell-1}(\delta_k+\tilde{\delta}_k).
	\end{aligned}
\end{equation}
Finally,
\begin{equation}
	E_\ell^c = F_\ell\cup\{\norm{S'_\ell}>B_\ell\}.
\end{equation}

Using Equation~\eqref{eq:activation_threshold} once more,
\begin{equation}
	\Pp(E_\ell^c) \le \sum_{k=1}^{\ell}(\delta_k+\tilde{\delta}_k).
\end{equation}
This proves the induction claim.

\paragraph{Conclusion.}
By~\eqref{eq:IH},
\begin{equation}
	\Pp(E_{L-1}) \ge 1-\sum_{k=1}^{L-1}(\delta_k+\tilde{\delta}_k) = 1-\varepsilon+\delta_L > 0
\end{equation}

Hence~\eqref{eq:cond_cheb} applies at the final layer. Let
\begin{equation}
	F_L:=\{\norm{\Delta_L}>A_L\}.
\end{equation}
The same event decomposition gives
\begin{equation}
	\begin{aligned}
		\Pp(F_L) &\le \Pp(\tilde Z_L>A_L\mid E_{L-1})\Pp(E_{L-1}) +\Pp(E_{L-1}^c) \\
		&\le \delta_L +\sum_{k=1}^{L-1}(\delta_k+\tilde{\delta}_k) \\
		&= \varepsilon
	\end{aligned}
\end{equation}
\hfill$\square$

\subsection{Proof of Lemma~\ref{lem:global} (global bound)}
\label{app:proof_global}

The recursion~\eqref{eq:global_recursion} is the sequential construction used in this proof.


\paragraph{Step 1: almost-sure domination.}
We show $\norm{\Delta_\ell} \le G_\ell$ for every $\ell$ by induction. At $\ell = 0$, $\norm{\Delta_0} = 0 = G_0$. For the step, by~\eqref{eq:Delta_le_Z}, the nonnegativity of $X_\ell$, and the induction hypothesis,
\begin{equation}
  \norm{\Delta_\ell}
  \;\le\; X_\ell \norm{\Delta_{\ell-1}} + Y_\ell \norm{S'_{\ell-1}}
  \;\le\; X_\ell\, G_{\ell-1} + Y_\ell \norm{S'_{\ell-1}}
  \;=\; G_\ell .
\end{equation}
This step is deterministic.

\paragraph{Step 2: one-sided Chebyshev.}
$G_L$ has finite mean $\muG$ and variance $\sigG^2$ by assumption,set $t^\star := \sigG \sqrt{(1-\delta)/\delta}$. If $\sigG > 0$, then $t^\star > 0$ and the one-sided Chebyshev inequality \citep{cantelli1928} gives
\begin{equation}
  \Pp(G_L > \muG + t^\star) \;\le\; \Pp(G_L - \muG \ge t^\star)
  \;\le\; \frac{\sigG^2}{\sigG^2 + t^{\star 2}} \;=\; \delta
\end{equation}
If $\sigG = 0$, then $G_L = \muG$ almost surely and $t^\star = 0$, so $\Pp(G_L > \muG + t^\star) = \Pp(G_L > \muG) = 0 \le \delta$; the non-strict event would fail in this degenerate case, which is why Lemma~\ref{lem:global} is stated with a strict inequality. In both cases, by Step~1, $\{\norm{\Delta_L} > \muG + t^\star\} \subseteq  \{G_L > \muG + t^\star\}$, and~\eqref{eq:cantelli} follows. \hfill$\square$

\subsection{Proof of the propagation rules of
Table~\ref{tab:dag_rules}}
\label{app:proof_dag}

\revadd{For the global bound, we represent a neural network, including self-attention operations, by a directed acyclic graph with node set $\mathcal V$ and edge set $\mathcal E$.} A node stores the output of an operation, and an edge $e:u\to v$ means that the operation at $v$ uses the output of $u$. The input and output nodes are $v_{\mathrm{in}}$ and $v_{\mathrm{out}}$. Nodes are evaluated after their parents.

At each node, the full-precision and quantized passes produce $S_v$ and $S'_v$, with $\Delta_v:=S_v-S'_v$. \revadd{We take Euclidean norms of the outputs viewed as vectors.} For a single-parent operation, $e:u\to v$ \revadd{indexes the parameters of the operation at $v$, such as $W_e$ for a linear layer}. Addition and concatenation have two parents, denoted by $u_1$ and $u_2$.

For a quantized linear or convolutional edge, define
\[
X_e:=\frac{\norm{W_e\Delta_u}}{\norm{\Delta_u}},\qquad
Y_e:=\frac{\norm{(W_e-W'_e)S'_u}}{\norm{S'_u}},
\]
with the same zero conventions as in Definition~\ref{def:amplification_ratios}.

Self-attention combines token features using weights computed from the tokens themselves~\citep{vaswani2017attention}. For example, one dot-product attention head computes
\begin{equation}
\begin{aligned}
Q&=SW_Q,\qquad K=SW_K,\qquad V=SW_V,\\
\operatorname{Att}(S)&=\operatorname{softmax}\!\left(\frac{QK^\top}{\sqrt{d_k}}\right)V.
\end{aligned}
\end{equation}
Each row of $S$ is a token, $d_k$ is the key dimension, and softmax is applied row by row. Multiple heads are computed in parallel, then concatenated and projected to the output space.

We write $F_W$ for the complete attention block, with $W$ collecting its quantizable parameters. The map $F_{W'}$ is the same block with quantized parameters. All other operations are unchanged. Define

\begin{equation}
\label{eq:att_ratios}
\begin{aligned}
X^{\mathrm{att}}
&:=\frac{\norm{F_W(S)-F_W(S')}}{\norm{S-S'}},
& L^{\mathrm{att}}
&:=\norm{F_W(S')-F_{W'}(S')},\\
G_{\mathrm{out}}
&:=X^{\mathrm{att}}G_{\mathrm{in}}+L^{\mathrm{att}},
\end{aligned}
\end{equation}
with $X^{\mathrm{att}}=0$ when $S=S'$.
For self-attention, use the edge-indexed form of
Equation~\eqref{eq:att_ratios}:
\[
X_e^{\mathrm{att}}
:=\frac{\norm{F_{W_e}(S_u)-F_{W_e}(S'_u)}}{\norm{\Delta_u}},\qquad L_e^{\mathrm{att}}
:=\norm{F_{W_e}(S'_u)-F_{W'_e}(S'_u)},
\]
with $X_e^{\mathrm{att}}=0$ when $\norm{\Delta_u}=0$. The realized normalization ratios below are also set to zero when $\norm{\Delta_u}=0$.

\begin{table}[t]
\caption{Propagation rules for the global bound $G$ on a computation DAG. If the bounds hold at the parent nodes, each update ensures
$\norm{\Delta_v}\le G_v$.}
\label{tab:dag_rules}
\centering
\setlength{\tabcolsep}{2pt}
\normalsize
\begin{tabular}{@{}p{0.26\linewidth}p{0.49\linewidth}p{\dimexpr0.25\linewidth-8pt\relax}@{}}
\toprule
Operation & Update of the bound $G$ & Justification \\
\midrule
Quantizable linear $S_v = W_e S_u$ &
  $G_v = X_e G_u + Y_e \norm{S'_u}$ &
  decomposition~\eqref{eq:exact_decomposition} \\
Frozen affine norm.\ $S_v = \Gamma_e S_u + \beta_e$ &
  $G_v = X^{\mathrm{BN}}_e G_u$, \ $X^{\mathrm{BN}}_e := \frac{\norm{\Gamma_e \Delta_u}}{\norm{\Delta_u}}$ &
  affine, frozen statistics \\
Data-dependent norm.\ $S_v = N(S_u)$ &
  $G_v = X^{\mathrm{N}}_e G_u$, \ $X^{\mathrm{N}}_e := \frac{\norm{N(S_u) - N(S'_u)}}{\norm{\Delta_u}}$ &
  realized ratio (GroupNorm, LayerNorm) \\
$L_\sigma$-Lipschitz pointwise $\sigma(\cdot)$ &
  $G_v = L_\sigma G_u$ &
  Lipschitz property \\
Residual addition $S_v = S_{u_1} + S_{u_2}$ &
  $G_v = G_{u_1} + G_{u_2}$ &
  triangle inequality \\
Channel concatenation &
  $G_v = \bigl(G_{u_1}^2 + G_{u_2}^2\bigr)^{1/2}$ &
  Pythagorean identity \\
NN-upsampling, factor $s$, $d$ spatial dims &
  $G_v = s^{d/2}\, G_u$ &
  replication identity
  \\
Additive embedding $S_v = S_u + W_e e$ &
  $G_v = G_u + \norm{(W_e - W'_e)\, e}$ &
  $e$ shared across the two passes (per sample) \\
Self-attention $S_v = F_{W_e}(S_u)$ &
  $G_v = X^{\mathrm{att}}_e G_u + L^{\mathrm{att}}_e$ &
  realized ratio and local term, Eq.~\eqref{eq:att_ratios} \\
\bottomrule
\end{tabular}
\end{table}


We prove that, on a computation directed acyclic graph (DAG) whose nodes are the operations of Table~\ref{tab:dag_rules}, the bound $G$ defined by the table rules satisfies $\norm{\Delta_v} \le G_v$ almost surely at every node $v$. The proof is a deterministic induction in topological order; the base case is the input node, where $\Delta_{v_{\mathrm{in}}} = 0 = G_{v_{\mathrm{in}}}$. For the inductive step, assume $\norm{\Delta_u} \le G_u$ at every parent $u$ of $v$; we treat each rule. Throughout, $S_v$ and $S'_v$ denote the exact and quantized outputs at $v$, and $\Delta_v = S_v - S'_v$.

\emph{(1) Quantizable linear.}
Identical to~\eqref{eq:Delta_le_Z}--\eqref{eq:Z_le_tildeZ} with $G_u$ in place of $A_{u}$:
$\norm{\Delta_v} \le X_e \norm{\Delta_u} + Y_e \norm{S'_u} \le X_e G_u + Y_e \norm{S'_u} = G_v$, using $X_e \ge 0$ and the induction hypothesis.

\emph{(2) Frozen affine normalization.}
The notation $\mathrm{BN}$ refers to batch normalization with fixed statistics; the same rule applies to any fixed affine normalization.
With frozen affine parameters $(\Gamma_e, \beta_e)$, the shift cancels in the difference: $\Delta_v = \Gamma_e \Delta_u$, so $\norm{\Delta_v} = X^{\mathrm{BN}}_e \norm{\Delta_u} \le X^{\mathrm{BN}}_e G_u$ by the definition of $X^{\mathrm{BN}}_e$ and its zero convention. 

\emph{(2b) Data-dependent normalization (GroupNorm, LayerNorm).}
Here $S_v = N(S_u)$ with statistics computed from the current activation, so $\Delta_v = N(S_u) - N(S'_u)$ is \emph{not} an affine image of $\Delta_u$. By the definition of the realized ratio $X^{\mathrm{N}}_e = \norm{N(S_u) - N(S'_u)}/\norm{\Delta_u}$,
$\norm{\Delta_v} = X^{\mathrm{N}}_e \norm{\Delta_u} \le
X^{\mathrm{N}}_e G_u$. The zero convention is safe on both sides: $\Delta_u = 0$ implies $S_u = S'_u$, hence $\Delta_v = 0$, since the normalization carries no quantized parameters in our setting.

\emph{(3) $L_\sigma$-Lipschitz pointwise nonlinearity.}
$\norm{\sigma(S_u) - \sigma(S'_u)} \le L_\sigma \norm{\Delta_u} \le L_\sigma G_u$.

\emph{(4) Residual addition.}
Here, $u_1$ and $u_2$ are the two parent nodes of $v$, whose outputs are added. $\Delta_v = \Delta_{u_1} + \Delta_{u_2}$, so by the triangle inequality and the induction hypothesis at both parents, $\norm{\Delta_v} \le \norm{\Delta_{u_1}} + \norm{\Delta_{u_2}} \le G_{u_1} + G_{u_2}$.

\emph{(5) Channel concatenation.} 
Here, $u_1$ and $u_2$ are the two parent nodes of $v$, whose outputs are concatenated along the channel dimension. $\norm{\Delta_v}^2 = \norm{\Delta_{u_1}}^2 + \norm{\Delta_{u_2}}^2 \le G_{u_1}^2 + G_{u_2}^2$, using $0 \le \norm{\Delta_{u_i}} \le G_{u_i}$; taking square roots gives the rule.

\emph{(6) Nearest-neighbor upsampling.}
Upsampling by factor $s$ in $d$ spatial dimensions replicates each entry of its input exactly $s^d$ times, and acts entrywise, so it commutes with the difference:
$\Delta_v = \mathrm{Up}_s(\Delta_u)$ and
\begin{equation}
  \norm{\mathrm{Up}_s(\Delta_u)}^2
  \;=\; s^d\, \norm{\Delta_u}^2,
\end{equation}
so $\norm{\Delta_v} = s^{d/2} \norm{\Delta_u} \le s^{d/2} G_u$. In two spatial dimensions the factor is $s$, i.e.\ $2$ for the standard $2 \times 2$ upsampling.

\emph{(7) Additive embedding.}  

An additive embedding supplies extra information to the network. A vector $e$, for example an encoding of the diffusion timestep, is projected by $W_e$ and added to the current output:
\begin{equation}
S_v=S_u+W_e e,\qquad S'_v=S'_u+W'_e e.
\end{equation}

The embedding input $e$ is identical on both forward passes for each sample; for timestep embeddings, $e$ varies across samples but is shared between the exact and quantized passes, which is all the rule needs. Hence $\Delta_v = \Delta_u + (W_e - W'_e)\, e$ and $\norm{\Delta_v} \le \norm{\Delta_u} + \norm{(W_e - W'_e)\, e} \le G_u + \norm{(W_e - W'_e)\, e}$; the local term is used in absolute form, so no normalization convention is needed.

\emph{(8) Self-attention (or any deterministic map with quantizable parameters, evaluated on the same inputs in both passes).}
Write $F_{W_e}$ for the block with exact parameters and $F_{W'_e}$ with quantized parameters. The error decomposes exactly as
\begin{equation} \label{eq:att_decomposition}
  \Delta_v
  \;=\; \underbrace{F_{W_e}(S_u) - F_{W_e}(S'_u)}_{\text{propagated}}
  \;+\; \underbrace{F_{W_e}(S'_u) - F_{W'_e}(S'_u)}_{\text{local}} .
\end{equation}
By the triangle inequality and the definitions~\eqref{eq:att_ratios}
(the zero convention on $X^{\mathrm{att}}_e$ is safe, since $\Delta_u = 0$ makes the propagated term vanish identically; the local term $L^{\mathrm{att}}_e$ is absolute and needs no convention),
\begin{equation}
  \norm{\Delta_v}
  \;\le\; X^{\mathrm{att}}_e \norm{\Delta_u}  + L^{\mathrm{att}}_e
  \;\le\; X^{\mathrm{att}}_e\, G_u + L^{\mathrm{att}}_e
  \;=\; G_v .
\end{equation}
No Lipschitz property of $F_{W_e}$ is used: $X^{\mathrm{att}}_e$ is the realized ratio on each input, so the domination follows directly from the exact decomposition in~\eqref{eq:att_decomposition}. Dot-product attention has no finite global Lipschitz constant on its unbounded input space \citep{kim2021lipschitz}, so $X^{\mathrm{att}}_e$ need not admit a useful almost-sure bound; the only additional assumption is the finite second moment of $G_{v_{\mathrm{out}}}$ required by
Lemma~\ref{lem:global}. When the block is kept in full precision ($W'_e = W_e$), the local term of~\eqref{eq:att_decomposition} vanishes and the rule reduces to $G_v = X^{\mathrm{att}}_e G_u$. \hfill$\square$

\subsection{A layerwise bound on DAGs}
\label{app:layerwise_dag}

For completeness, we extend Theorem~\ref{thm:layerwise} to the DAG defined above. At each node $v$, the output bound $B_v$ is a deterministic constant satisfying
\begin{equation}
\Pp(\norm{S'_v}\le B_v)\ge 1-\tilde{\delta}_v.
\end{equation}
These bounds are assumed to be available. The recursion below constructs the error bounds $A_v$.

Start with $A_{v_{\mathrm{in}}}=0$ and $B_{v_{\mathrm{in}}}=B_0$, and process nodes after their parents. For a parent $u$, define
\begin{equation}
E_u:=\{\norm{\Delta_u}\le A_u\}\cap\{\norm{S'_u}\le B_u\}.
\end{equation}
Let $\mathcal V_{\mathrm{prob}}$ contain the nodes where we apply conditional Cantelli: quantizable linear layers and frozen affine or data-dependent normalization layers.  For such a node $v$ with parent $u$ and edge $e:u\to v$, define
\begin{equation}
\tilde Z_v=
\begin{cases}
X_e A_u+Y_e B_u, & \text{quantizable linear layer},\\
X_e^{\mathrm{BN}}A_u, & \text{frozen affine normalization},\\
X_e^{\mathrm{N}}A_u, & \text{data-dependent normalization}.
\end{cases}
\end{equation}
We then set
\begin{equation}
A_v=\E[\tilde Z_v\mid E_u]
+\sqrt{\frac{1-\delta_v}{\delta_v}\Var(\tilde Z_v\mid E_u)}.
\end{equation}
This is the construction in Equation~\eqref{eq:def_A}, applied to a node rather than a layer. The proof below establishes that each conditioning event has positive probability.

The other nodes use deterministic updates:
\begin{equation}
A_v=
\begin{cases}
L_\sigma A_u, & \text{pointwise nonlinearity},\\
A_{u_1}+A_{u_2}, & \text{addition},\\
(A_{u_1}^2+A_{u_2}^2)^{1/2}, & \text{concatenation},\\
s^{d/2}A_u, & \text{nearest-neighbor upsampling},\\
A_u+L_e, & \text{additive embedding},
\end{cases}
\qquad L_e:=\norm{(W_e-W'_e)e}.
\end{equation}
For the additive-embedding update, we assume that $e$ is fixed across samples, so $L_e$ is a deterministic constant. The global rule only requires the two passes to share $e$ for each sample; it allows $e$ to vary between samples.

\begin{lemma}[Layerwise bound on DAGs] \label{lem:layerwise_dag}
Assume 
\begin{itemize}
\item the DAG is built from  operations from
Table~\ref{tab:dag_rules} excluding self-attention (see the closing remark); 
\item for each node $v$, the levels $\delta_v\in(0,1)$ and $\tilde{\delta}_v\in[0,1)$ are fixed, and $B_v$ is a deterministic bound satisfying $\Pp(\norm{S'_v} \le B_v) \ge 1 - \tilde{\delta}_v$, with $\tilde{\delta}_{v_{\mathrm{in}}} = 0$ under the bounded-support convention;
\item \revadd{for each $v\in\mathcal V_{\mathrm{prob}}$ with parent $u$, $\tilde Z_v$ has finite conditional mean and variance given $E_u$}; 
\item $\sum_{v \in \mathcal{V}_{\mathrm{prob}}} \delta_v + \sum_{v \in \mathcal{V}} \tilde{\delta}_v < 1$, which by the invariant~\eqref{eq:invariant} keeps every conditioning event of positive probability.
\end{itemize}
 Then,  defining events
$E_v := \{\norm{\Delta_v} \le A_v\} \cap \{\norm{S'_v} \le B_v\}$ for all nodes $v$, we have
\begin{equation}
  \Pp\bigl( \norm{\Delta_{v_{\mathrm{out}}}} > A_{v_{\mathrm{out}}} \bigr)
  \;\le\;  \sum_{v \in \mathcal{V}_{\mathrm{prob}}} \delta_v
  \;+\; \sum_{v \in \mathcal{V}} \tilde{\delta}_v .
\end{equation}
\end{lemma}

\begin{proof}
For each \revadd{node $v$ where conditional Cantelli is applied} with parent $u$, define the
\emph{local failure event}
\begin{equation} \label{eq:def_F_v}
  F_v := \{\tilde Z_v > A_v\} \cap E_u,
\end{equation}

Here $\tilde Z_v$ is the comparison variable defined above. Its local failure probability is bounded during the induction below. For each node $v$, define the \emph{output failure event} $H_v := \{\norm{S'_v} > B_v\}$, with $\Pp(H_v) \le \tilde{\delta}_v$ by the assumed bound on $B_v$.

We prove, by induction in topological order, the set inclusion
\begin{equation} \label{eq:invariant}
  E_v^c \;\subseteq\;
  \bigcup_{w \in \mathcal{V}_{\mathrm{prob}},\, w \preceq v} F_w
  \;\cup\; \bigcup_{w \preceq v} H_w ,
\end{equation}
where $w \preceq v$ ranges over the ancestors of $v$ (including $v$). Because the right-hand side is a union of \emph{sets}, an ancestor shared by several branches contributes a single set, and no double counting can occur at merge nodes; the union bound is paid exactly once, at the output.

\emph{Base case.} At the input node,
$\Delta_{v_{\mathrm{in}}} = 0 = A_{v_{\mathrm{in}}}$ deterministically, so $E_{v_{\mathrm{in}}}^c = H_{v_{\mathrm{in}}}$ and~\eqref{eq:invariant} holds.

\emph{Inductive step.} 

Assume that the invariant and the local failure bounds hold at all previously processed nodes. For each parent $u$,
\begin{equation}
\Pp(E_u)\ge
1-\sum_{\substack{w\in\mathcal V_{\mathrm{prob}}\\w\preceq u}}\delta_w
-\sum_{w\preceq u}\tilde{\delta}_w>0.
\end{equation}
The last inequality follows from the total probability budget in the lemma. Thus the conditional moments used at $v$ are well defined.

\emph{A single-parent node using conditional Cantelli.}
On $E_u$, we have $\norm{\Delta_v}\le\tilde Z_v$. Hence
\begin{equation}
\{\norm{\Delta_v}>A_v\}\cap E_u
\subseteq\{\tilde Z_v>A_v\}\cap E_u=F_v.
\end{equation}
Conditional Cantelli gives
\begin{equation}
\Pp(F_v)=\Pp(\tilde Z_v>A_v\mid E_u)\Pp(E_u)\le\delta_v,
\end{equation}
including the zero-variance case as in~\eqref{eq:cond_cheb}. Since $E_v^c=\{\norm{\Delta_v}>A_v\}\cup H_v$,
\begin{equation}
E_v^c\subseteq F_v\cup E_u^c\cup H_v.
\end{equation}
The induction hypothesis at $u$ proves the invariant at $v$.

\emph{A single-parent node using a deterministic update.}
Write the update as $A_v=cA_u+a$. For a pointwise nonlinearity, upsampling, and additive embedding, respectively,
\begin{equation}
(c,a)=(L_\sigma,0),\qquad (c,a)=(s^{d/2},0),\qquad (c,a)=(1,L_e).
\end{equation}
The embedding case uses the deterministic $e$ assumed above. On $E_u$,
\begin{equation}
\norm{\Delta_v}\le c\norm{\Delta_u}+a\le cA_u+a=A_v.
\end{equation}
Therefore,
\begin{equation}
E_v^c\subseteq E_u^c\cup H_v.
\end{equation}
The induction hypothesis proves the invariant at $v$. There is no local event $F_v$ in this case.

\emph{An addition or concatenation node.}
On $E_{u_1}\cap E_{u_2}$, addition gives
\begin{equation}
\norm{\Delta_v}\le\norm{\Delta_{u_1}}+\norm{\Delta_{u_2}}
\le A_{u_1}+A_{u_2}=A_v.
\end{equation}
For concatenation,
\begin{equation}
\norm{\Delta_v}^2=\norm{\Delta_{u_1}}^2+\norm{\Delta_{u_2}}^2
\le A_{u_1}^2+A_{u_2}^2=A_v^2.
\end{equation}
Thus, in both cases,
\begin{equation}
E_v^c\subseteq E_{u_1}^c\cup E_{u_2}^c\cup H_v.
\end{equation}
Every ancestor of either parent is an ancestor of $v$. Applying the induction hypothesis at both parents and adding $H_v$ proves the invariant. Shared ancestors appear only once in the union.

\emph{Conclusion.} Applying~\eqref{eq:invariant} at the output node, the inclusion $\{\norm{\Delta_{v_{\mathrm{out}}}} > A_{v_{\mathrm{out}}}\} \subseteq E_{v_{\mathrm{out}}}^c$, and the union bound once,
\begin{equation}
  \Pp\bigl(\norm{\Delta_{v_{\mathrm{out}}}} > A_{v_{\mathrm{out}}}\bigr)
  \;\le\; \sum_{w \in \mathcal{V}_{\mathrm{prob}}} \Pp(F_w) + \sum_{w \in \mathcal{V}} \Pp(H_w) \;\le\;
  \sum_{w \in \mathcal{V}_{\mathrm{prob}}} \delta_w + \sum_{w \in \mathcal{V}} \tilde{\delta}_w . \qedhere
\end{equation}
\end{proof}

\begin{remark}[Scope]
	Lemma~\ref{lem:layerwise_dag} covers the operations in Table~\ref{tab:dag_rules} except self-attention. At a self-attention node, the comparison variable $X^{\mathrm{att}}_e A_u + L^{\mathrm{att}}_e$ dominates $\norm{\Delta_v}$ on $E_u$ by the same argument, so the lemma extends at the cost of finite conditional moments of $(X^{\mathrm{att}}_e, L^{\mathrm{att}}_e)$ given $E_u$, including their conditional covariance, the comparison variable being the linear combination $X^{\mathrm{att}}_e A_u + L^{\mathrm{att}}_e$; we do not use this extension in our experiments, where attention graphs are handled by the global bound of Lemma~\ref{lem:global}.
\end{remark}

\subsection{Sensitivity of the empirical functionals}
\label{app:functional}

Let $n\ge2$ and let $P_1, \dots, P_n \in [0, M]$ be the calibration values of $P_{g,b}$ and replace one sample $P_j$ by an arbitrary
$P'_j \in [0, M]$. The empirical mean moves by
$|\hat\mu' - \hat\mu| = |P'_j - P_j| / n \le M/n$, a bound attained when one sample jumps across the full range. For the standard deviation with divisor $n-1$, $\hat\sigma^2 = \frac{n}{n-1}(\frac1n \sum_i P_i^2 - \hat\mu^2)$ gives $|\hat\sigma'^2 - \hat\sigma^2| \le 3M^2/(n-1)$, and $|\sqrt{a} - \sqrt{b}| \le \sqrt{|a - b|}$ yields
$|\hat\sigma' - \hat\sigma| \le M \sqrt{3/(n-1)}$. The empirical maximum admits no such bound: with $P_1 = M$ and $P_2 = \dots = P_n = 0$, replacing $P_1$ by $0$ moves the maximum by $M$ for every $n$.

At the level of the score~\eqref{eq:local_score} the two rates combine: the worst-case displacement of $\sloc_g(b)$ is at most $M/n + (M/\sqrt{\delta_{\mathrm{score}}})\sqrt{3/(n-1)}$, dominated by the standard-deviation term scaled by $1/\sqrt{\delta_{\mathrm{score}}}$. At the operating
point of our experiments ($n = 256$, $\delta_{\mathrm{score}} = 0.05$) this evaluates to about $0.49\,M$, against $M$ for the maximum; the asymptotic separation in $n$ is therefore partly consumed by the $1/\sqrt{\delta_{\mathrm{score}}}$ factor at practical sample sizes. This replacement calculation concerns one changed value; it is not a stability guarantee for the selected allocation.


%% file: appendix/experiments_appendix.tex
\section{Experimental details and complete controls}
\label{app:experiments}

\subsection{DRUNet implementation details}
\label{app:drunet_details}

The $36$ allocation blocks contain $64$ convolutions: a one-convolution head; three downsampling stages with four residual blocks and one downsampling convolution each; four residual blocks in the body; three upsampling stages with one upsampling convolution and four residual blocks each; and a one-convolution tail. Activations stay in FP32. Weight quantization is symmetric and per output channel within the RTN comparisons. The AIMET comparison uses the common TF-Enhanced backend described in Section~\ref{sec:common_setup}. It runs both official AMP selection phases, adapted to our blocks and weight-memory budget. The selection uses calibration SQNR, not the test set.

For every data-dependent method, calibration seeds 0, 1, and 2 select $256$ images each. The fixed test set contains $100$ images selected with seed 42. The input denoising noise has standard deviation $0.1$. The clean target, test input, and evaluation code are shared within every comparison. The reported average bit-widths are weighted by the number of weights in each block.

\paragraph{Confidence intervals.}
For each test image, we compute the PSNR difference between the two methods for each of the three calibration runs. We then average the three differences. This gives $100$ values, one per test image. The $95\%$ confidence interval for their mean is $\bar d\pm c\,s_d/\sqrt{100}$, where $s_d$ is their sample standard deviation. CLADO uses the normal critical value $c=1.960$. The reported AIMET intervals are reproduced with the Student critical value $c=1.984$ ($99$ degrees of freedom). These intervals do not account for new calibration sets.

\begin{table}[h]
\caption{DRUNet comparison settings. The block partition and
parameter-weighted budget are shared within every \revadd{baseline} comparison.}
\label{tab:external_protocols}
\centering
\setlength{\tabcolsep}{3.5pt}
\normalsize
\begin{tabular}{@{}p{0.30\linewidth}p{0.20\linewidth}p{0.15\linewidth}p{\dimexpr0.35\linewidth-21pt\relax}@{}}
\toprule
Reference & Quantizer & \shortstack[c]{Candidate\\bit-widths} & Allocation rule \\
\midrule
OMPQ-ORM \citep{ma2023ompq} & per-channel RTN & 3--8 & adapted ORM + MILP \\
CLADO \citep{deng2023clado} & per-channel RTN & 3, 4, 8 & pairwise IQP \\
AIMET AMP \citep{pandey2023practical} & TF-Enhanced & 4, 6, 8 & sensitivity + Pareto \\
Adapted output-SQNR \citep{pandey2023practical} & per-channel RTN & 3--8 & output SQNR + greedy \\
\bottomrule
\end{tabular}
\end{table}

\subsection{Clean output-SQNR control for different budgets}
\label{app:sqnr_control}
\label{app:clean_sweeps}

This clean control adapts the output-SQNR sensitivity used by
\citet{pandey2023practical} to the common DRUNet blocks and weight-only RTN backend. One block is quantized at a time, the final DRUNet output is compared with the FP32 output, and the resulting negative SQNR is passed to the same greedy allocation rule. It is a score control, not the complete AIMET AMP pipeline. Writing $f_{W^{g\to b}}$ for the network in which only block $g$ is quantized to $b$ bits, the implemented score is
\[
\operatorname{SQNR}_{g,b}=10\log_{10}
\frac{\sum_i\lVert f_W(x_i)\rVert_F^2}
{\max\!\left\{\sum_i\lVert f_W(x_i)-f_{W^{g\to b}}(x_i)\rVert_F^2,
\epsilon_{64}\right\}},
\qquad s^{\mathrm{SQNR}}_{g,b}=-\operatorname{SQNR}_{g,b},
\]
where the sums cover all calibration images and output elements, and $\epsilon_{64}$ is the smallest positive \texttt{float64} value.

\begin{table}[h]
	\caption{PTQ with different budgets. Entries are PSNR in dB, with realized average bit-widths in parentheses. Our method, output-SQNR, and OMPQ-ORM use W3--W8; CLADO uses W3/W4/W8.}
	\label{tab:sqnr_control}
	\centering
\setlength{\tabcolsep}{4pt}
	\normalsize
	\begin{tabular}{ccccc}
		\toprule
		$B$ & Ours & Output-SQNR & OMPQ-ORM & CLADO \\
		\midrule
		3.5 & 33.300 (3.483) & 31.957 (3.494) & 31.970 (3.494) & 34.101 (3.494) \\
		4.0 & 34.401 (3.987) & 34.247 (3.988) & 32.587 (3.988) & 34.409 (3.988) \\
		4.5 & 34.512 (4.439) & 34.446 (4.494) & 33.023 (4.494) & 34.503 (4.458) \\
		5.0 & 34.574 (4.895) & 34.460 (4.964) & 34.453 (4.964) & 34.578 (4.992) \\
		\bottomrule
	\end{tabular}
\end{table}
\subsection{CLADO control with identical candidate bit-widths}
\label{app:candidate_matched}

The main CLADO comparison uses each algorithm in its intended candidate space: W3/W4/W8 for CLADO and W3--W8 for our method. The corrected control in Table~\ref{tab:clado_b348} also restricts our algorithm to W3/W4/W8. Each move is ranked by score increase per weight bit saved, using the same denominator $(b_g-b_g^{-})n_g$ as in Section~\ref{sec:greedy}. Thus only the candidate set changes. Both methods share a target ceiling, but their realized precision can still differ.

\begin{table}[h]
\caption{Clean CLADO control with the same W3/W4/W8 candidates for both
methods.}
\label{tab:clado_b348}
\centering
\setlength{\tabcolsep}{3.5pt}
\normalsize
\begin{tabular}{ccccc}
\toprule
$B$ & \shortstack[c]{Our mean\\bit-width} & Our PSNR & \shortstack[c]{CLADO mean\\bit-width} & CLADO PSNR \\
\midrule
3.5 & 3.4704 & 32.9967 & 3.4935 & 34.1012 \\
4.0 & 3.9684 & 34.4081 & 3.9884 & 34.4089 \\
4.5 & 4.4021 & 34.4629 & 4.4583 & 34.5028 \\
5.0 & 4.9803 & 34.5139 & 4.9923 & 34.5776 \\
\bottomrule
\end{tabular}
\end{table}

Table~\ref{tab:clado_b348_corrupt} uses the same candidates at $B=4$ under the two corruptions. The clean values are almost identical. Under corrupted calibration, our allocation keeps the same PSNR while CLADO loses more than $5$~dB. This control shows that the corruption result is not created by giving our method intermediate bit-widths.

\begin{table}[h]
\caption{CLADO control at target $B=4$ with W3/W4/W8 candidates for both methods. Positive differences favor our method.}
\label{tab:clado_b348_corrupt}
\centering
\setlength{\tabcolsep}{3.2pt}
\normalsize
\begin{tabular}{lccccc}
\toprule
Calibration & \shortstack[c]{Our mean\\bit-width} & Our PSNR & \shortstack[c]{CLADO mean\\bit-width} & CLADO PSNR & Our gain \\
\midrule
clean & 3.9684 & 34.4081 & 3.9884 & 34.4089 & $-0.0008$ \\
$p=0.5$ shared & 3.9684 & 34.4081 & 3.9392 & 27.8235 & $+6.5845$ \\
$p=0.9$ per-channel & 3.9684 & 34.4081 & 3.9870 & 28.9281 & $+5.4800$ \\
\bottomrule
\end{tabular}
\end{table}

\subsection{Measured selection time}
\label{app:runtime}

We measure the complete mixed-precision selection stage on DRUNet with $256$ clean calibration images. Table~\ref{tab:recorded_runtime} reports median times over three runs after warm-up. Each comparison uses the
backend and candidate bit-widths of the reference method. Model loading, data preparation, quantization application, and PSNR evaluation are excluded.

\begin{table}[h]
\caption{Median DRUNet mixed-precision selection times over three runs.
Speed-up is the reference time divided by our matched time.}
\label{tab:recorded_runtime}
\centering
\setlength{\tabcolsep}{4pt}
\normalsize
\begin{tabular}{lccc}
\toprule
Matched comparison & Our time & Reference time & Speed-up \\
\midrule
AIMET backend & \textbf{1 min 29 s} & AIMET AMP: 42 min 04 s & $28.4\times$ \\
OMPQ, fixed $\beta$ & \textbf{15.7 s} & OMPQ-ORM: 8 min 01 s & $30.6\times$ \\
OMPQ, including $\beta$ search & \textbf{15.7 s} & OMPQ-ORM: 13 min 51 s & $52.9\times$ \\
CLADO W3/W4/W8 & \textbf{8.66 s} & CLADO: 6 h 11 min & $2{,}570\times$ \\
\bottomrule
\end{tabular}
\end{table}

All methods use one NVIDIA A100-PCIE-40GB, except CLADO, which uses two. The reported CLADO speed-up compares the time required to complete the selection.

\begin{reviewaddition}
\paragraph{Layerwise allocation.}
We add the recorded times for collecting local perturbations, computing scores and assigning bit-widths. We count the full collection shared with the blockwise variant and exclude PSNR evaluation. The median over three calibration seeds is $24.32$ seconds on one A100. These measurements are separate from the benchmark in Table~\ref{tab:recorded_runtime}.
\end{reviewaddition}

\subsection{Complete DRUNet curves and allocation stability}
\label{app:full_curves}

Figure~\ref{fig:clado_corruption_curves} reports the  CLADO comparison, including a control with the same candidate bit-widths.

Figures~\ref{fig:clado_stability}, \ref{fig:ompq_curves},
\ref{fig:ompq_stability}, \ref{fig:aimet_curves}, and
\ref{fig:aimet_stability} give the corresponding allocation diagnostics and the complete OMPQ and AIMET comparisons. In the recorded $\sigma_{\mathrm{in}}=0.20$ AIMET condition, the input-noise standard deviation changes from $0.10$ to $0.20$; the block partition, calibration size, candidate bit-widths, and evaluation procedure are unchanged.

\begin{figure}[!t]
\centering
\includegraphics[width=\textwidth]{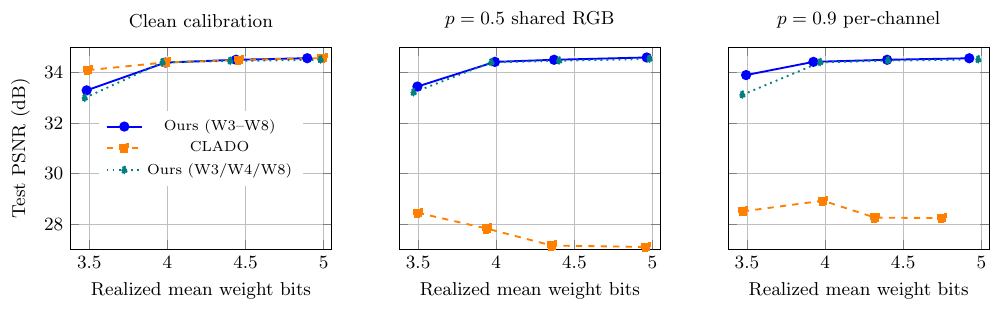}
\caption{DRUNet PSNR versus average bit-width for the adapted CLADO comparison. \revadd{Higher PSNR is better.} The dotted curve uses W3/W4/W8 for both methods. Under both calibration corruptions, our method achieves much higher PSNR than CLADO, even with the same candidate bit-widths.}
\label{fig:clado_corruption_curves}
\end{figure}

\begin{figure}[!t]
\centering
\includegraphics[width=0.86\textwidth]{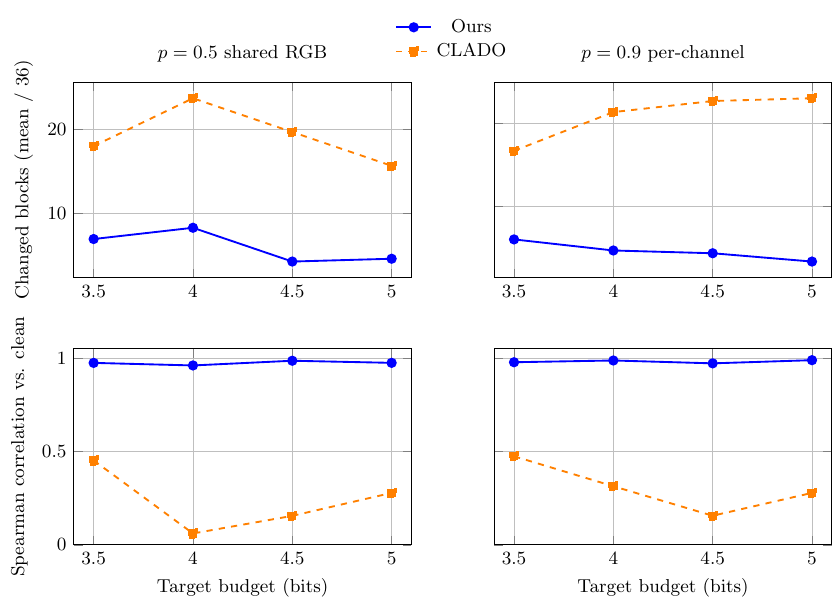}
\caption{Allocation stability under corrupted calibration for our method and CLADO. Top: changed blocks \revadd{(lower is better)}. Bottom: Spearman correlation with the configuration $\boldsymbol{b}$ calculated with a clean dataset \revadd{(higher is better)}. Values are averaged over three calibration seeds. Our method changes fewer block bit-widths than CLADO and better preserves their ordering under both calibration corruptions.}
\label{fig:clado_stability}
\end{figure}

\begin{figure}[!t]
\centering
\includegraphics[width=\textwidth]{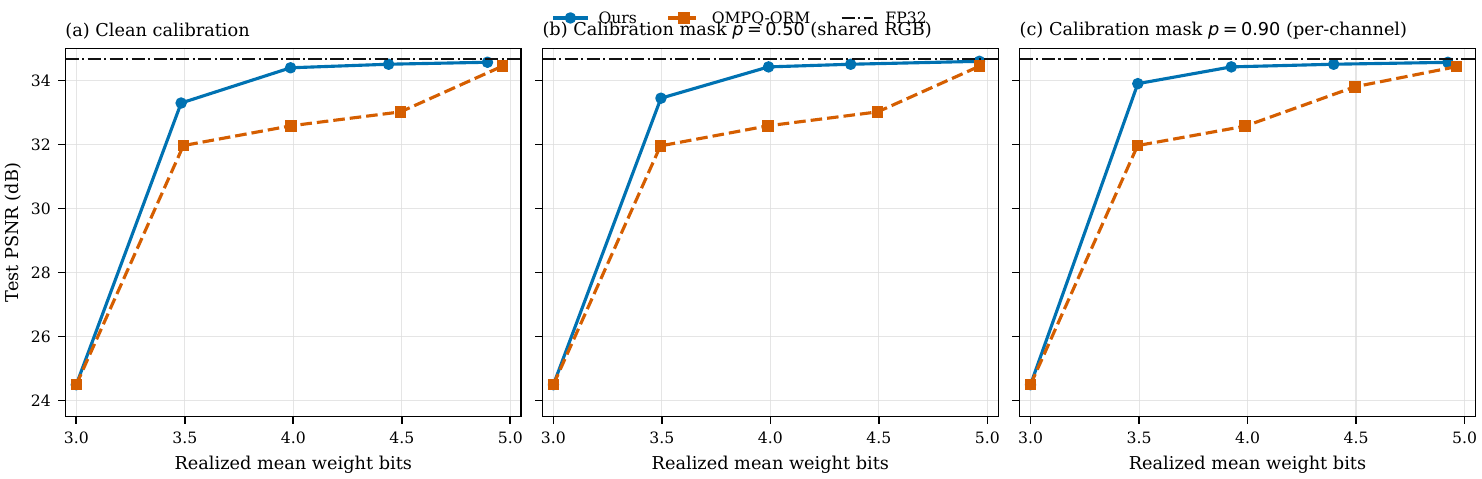}
\caption{DRUNet PSNR versus average bit-width for the adapted OMPQ-ORM comparison under clean and corrupted calibration. \revadd{Higher PSNR is better.} Our method achieves higher PSNR at intermediate budgets under both clean and corrupted calibration.}
\label{fig:ompq_curves}
\end{figure}

\begin{figure}[!t]
\centering
\includegraphics[width=0.88\textwidth]{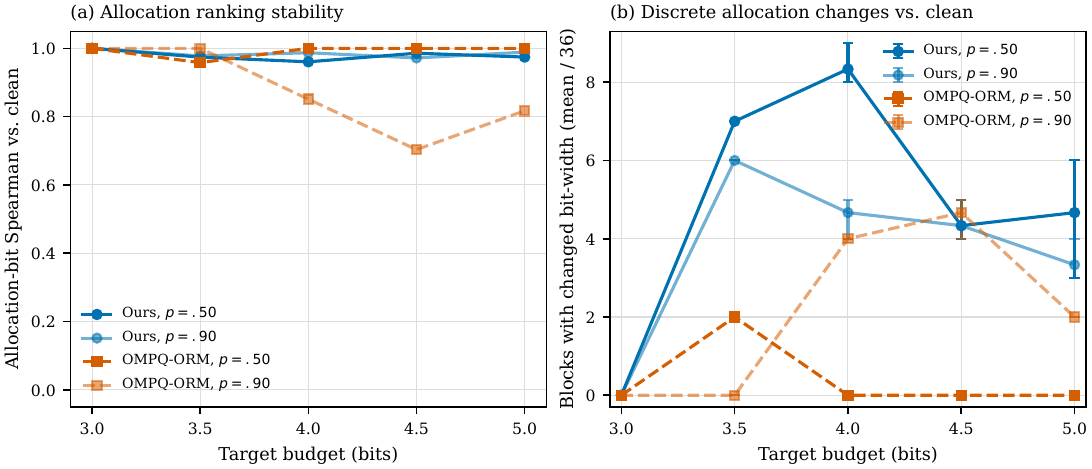}
\caption{Allocation stability for the adapted OMPQ-ORM comparison under corrupted calibration. \revadd{Left: Spearman correlation (higher is better). Right: changed blocks (lower is better).} OMPQ-ORM often changes fewer block bit-widths, but its bit-width ordering is less stable under per-channel corruption at higher budgets.}
\label{fig:ompq_stability}
\end{figure}

\begin{figure}[!t]
\centering
\includegraphics[width=0.90\textwidth]{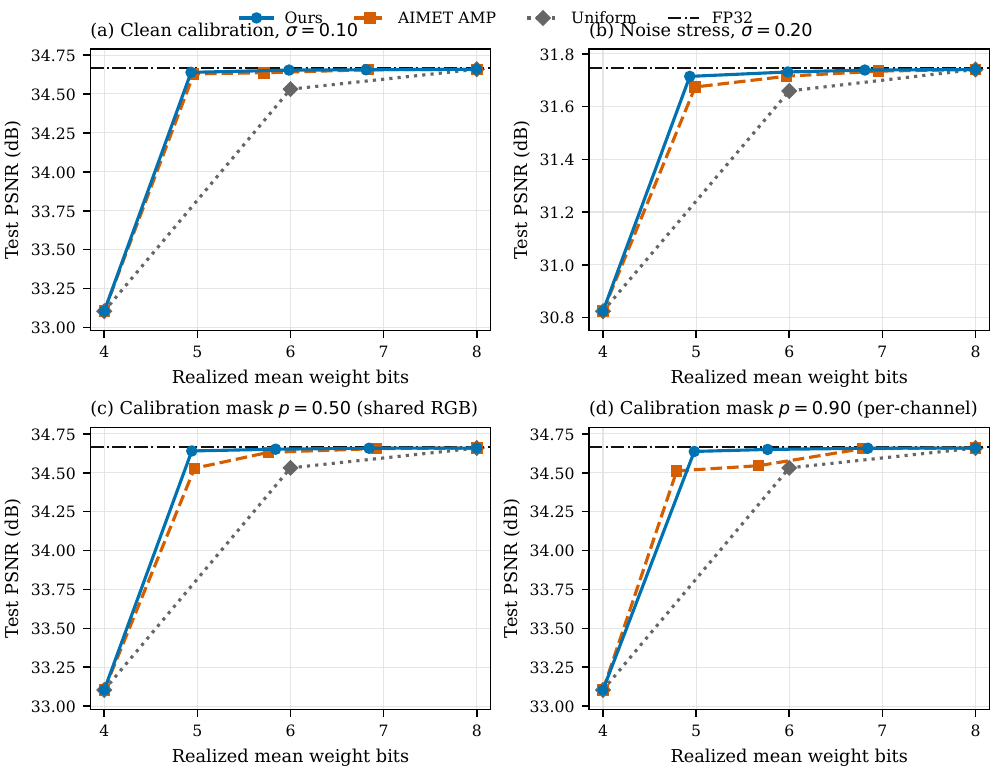}
\caption{DRUNet PSNR versus average bit-width for AIMET under clean and corrupted calibration and at $\sigma_{\mathrm{in}}=0.20$. \revadd{Higher PSNR is better.} Our method has an advantage at intermediate budgets under corrupted calibration, while both methods approach FP32 PSNR at higher budgets.}
\label{fig:aimet_curves}
\end{figure}

\begin{figure}[!t]
\centering
\includegraphics[width=0.88\textwidth]{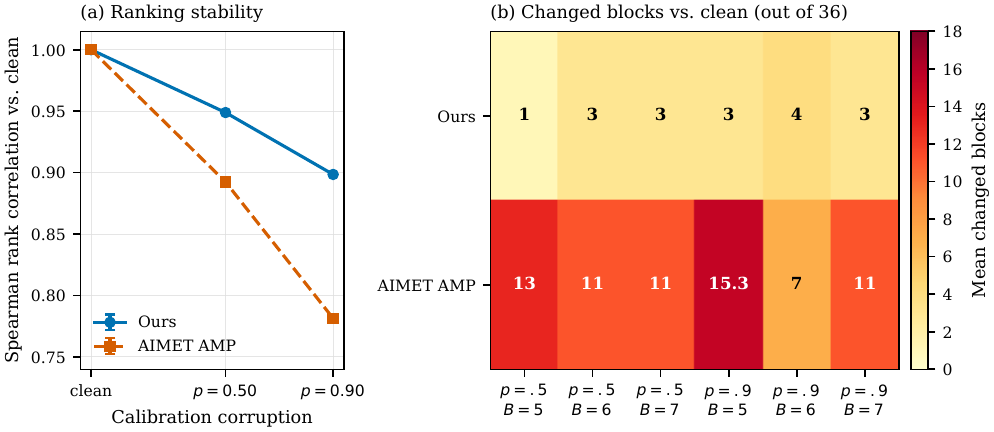}
\caption{AIMET allocation stability under the conditions in
Figure~\ref{fig:aimet_curves}. \revadd{Left: Spearman correlation (higher is better). Right: changed blocks (lower is better).} Our method changes fewer block bit-widths than AIMET and better preserves their ordering under both calibration corruptions.}
\label{fig:aimet_stability}
\end{figure}

\clearpage
\subsection{Fixed-random qualitative examples}
\label{app:qualitative}

The following panels use indices selected with fixed random seeds before inspecting method differences. They are illustrations only; all numerical conclusions use the same $100$ test images and three calibration sets as the quantitative tables. Images are clipped to $[0,1]$ for display, while PSNR is computed from the unclipped model outputs. The average bit-width labels in the panels describe the displayed calibration run, whereas the quantitative tables average three runs. Figures~\ref{fig:clado_qualitative},
\ref{fig:ompq_qualitative}, and~\ref{fig:aimet_qualitative} show the three comparisons.

\begin{figure}[!t]
\centering
\includegraphics[width=\textwidth]{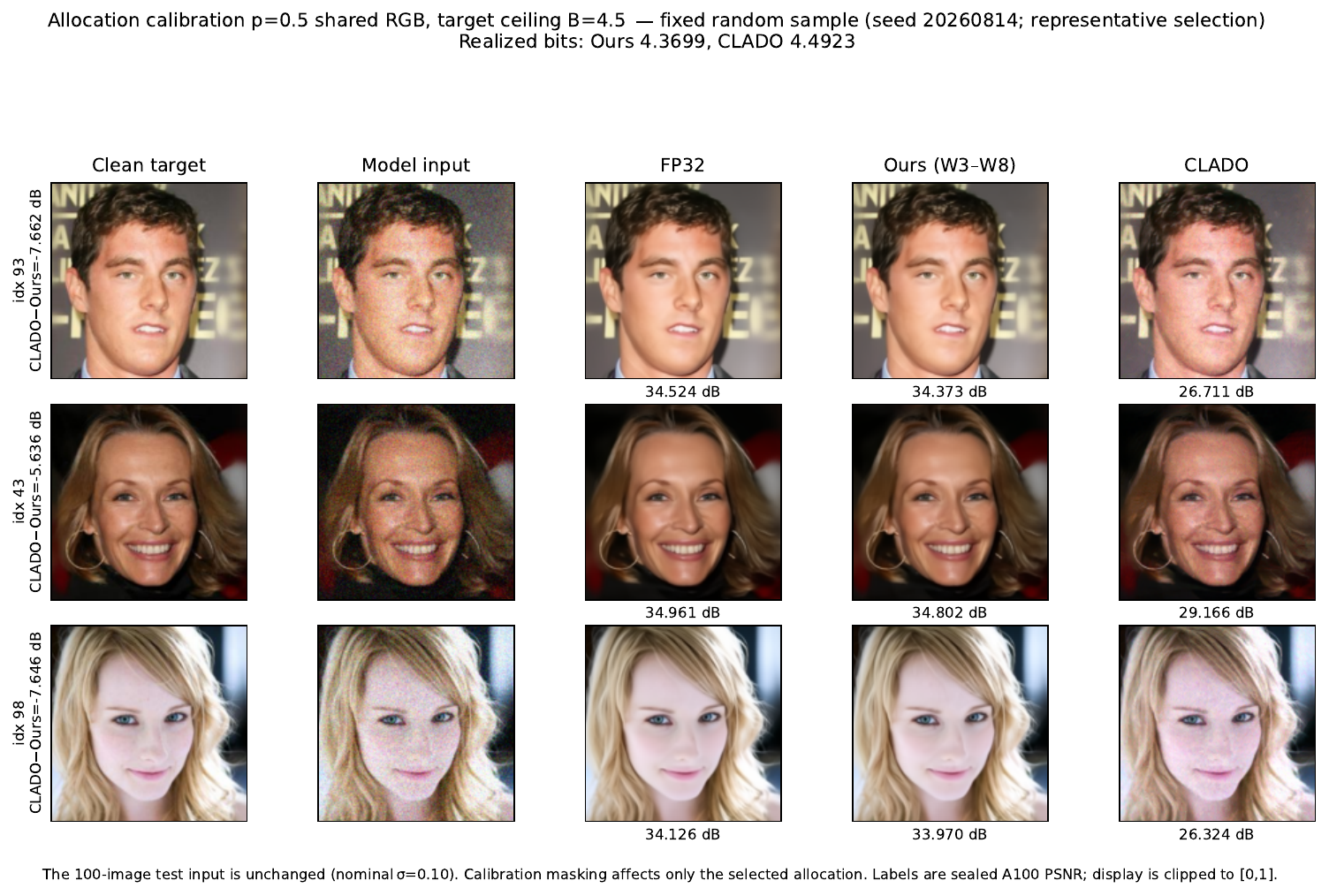}
\caption{Fixed-random DRUNet examples for the adapted CLADO comparison at $B=4.5$ under the $p=0.5$ shared calibration corruption. Test inputs are unchanged, and the displayed allocation uses calibration seed 0. In these examples, our outputs remain close to FP32, while CLADO leaves visibly more residual noise.}
\label{fig:clado_qualitative}
\end{figure}

\begin{figure}[!t]
\centering
\includegraphics[width=\textwidth]{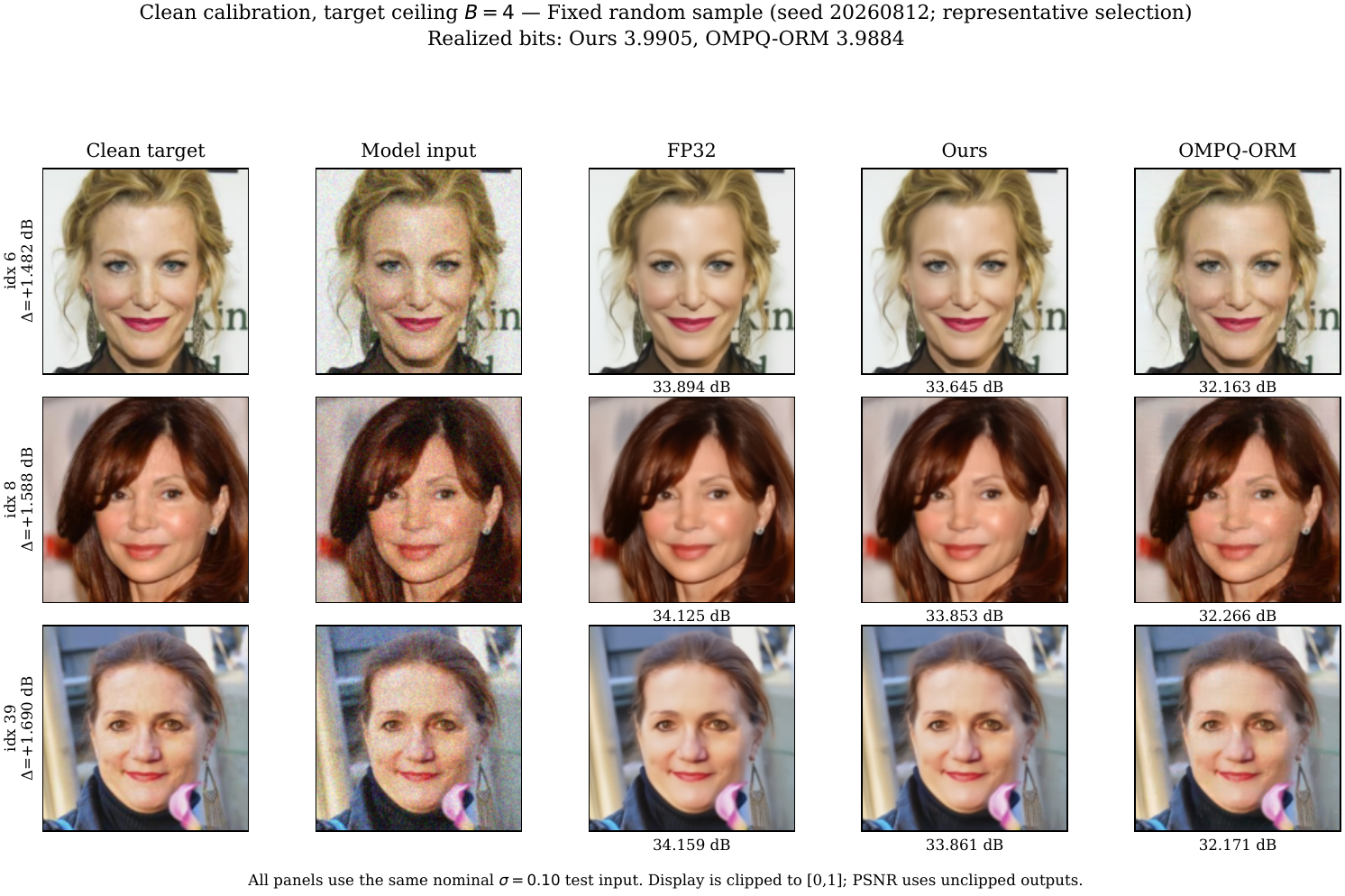}
\caption{Fixed-random DRUNet examples for the adapted OMPQ-ORM comparison at $B=4$ under clean calibration. Our outputs remain close to FP32 and have higher PSNR than OMPQ-ORM in these examples.}
\label{fig:ompq_qualitative}
\end{figure}

\begin{figure}[!t]
\centering
\includegraphics[width=\textwidth]{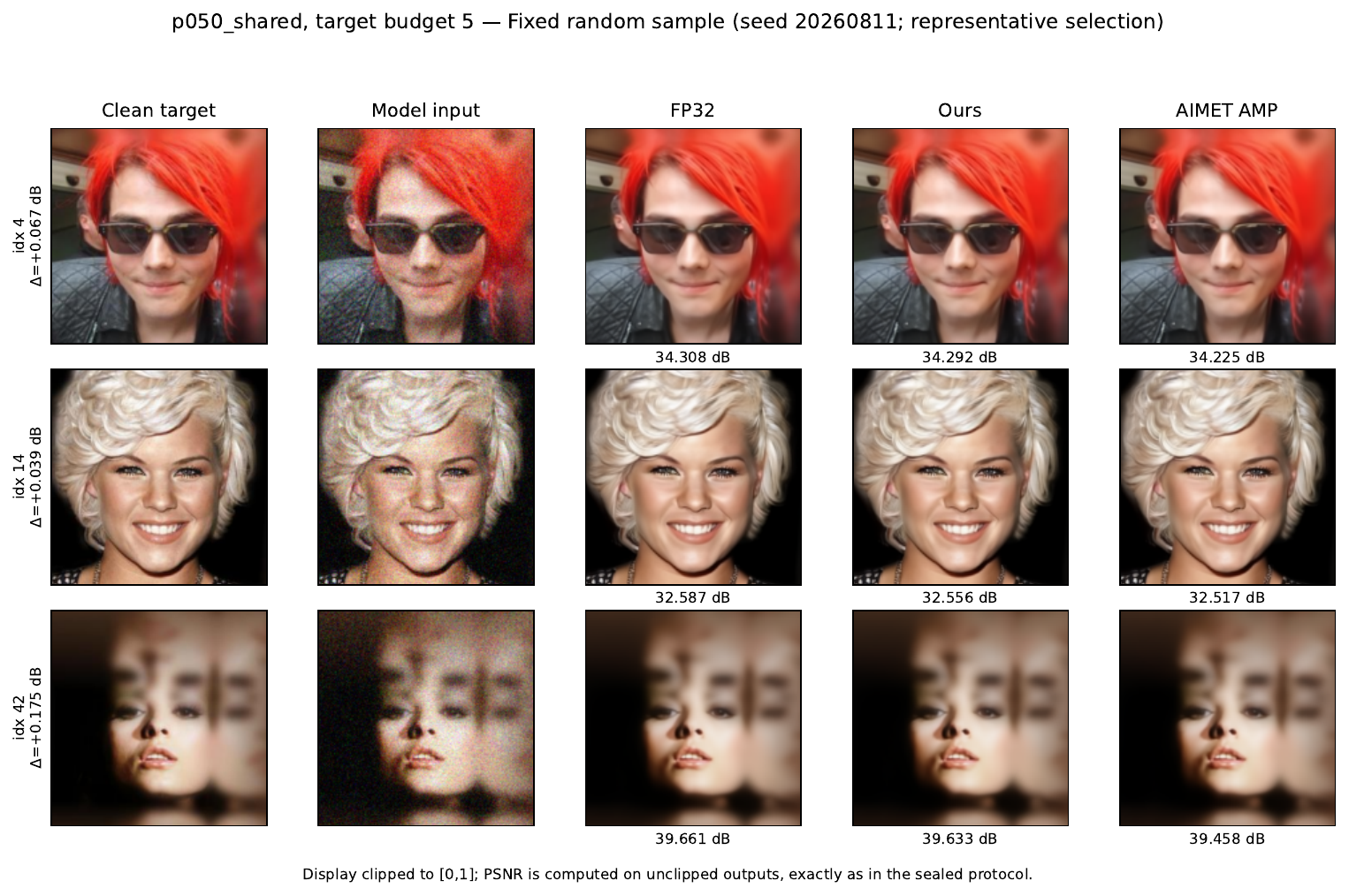}
\caption{Fixed-random DRUNet examples for AIMET at $B=5$ under the $p=0.5$ shared calibration corruption. Both methods remain visually close to FP32, with slightly higher PSNR for our method in these examples.}
\label{fig:aimet_qualitative}
\end{figure}

\clearpage
\subsection{Latent-diffusion protocol and metric scope}
\label{app:ldm_transfer}

We use the pretrained CompVis CelebA-HQ $256\times256$ latent diffusion model \citep{karras2018progressive,rombach2022ldm}. Weight reconstruction for our method and Q-Diffusion uses the Q-Diffusion/AdaRound backend \citep{li2023qdiffusion}. The local score and greedy allocation rule are unchanged from the DRUNet experiment.

Our method and Q-Diffusion use the same $5{,}120$ calibration timestep--latent pairs: $256$ pairs at each of $20$ timesteps, generated with seed 1234. The U-Net has $123$ quantized weight tensors grouped into $48$ allocation blocks. Our allocation uses candidates W3--W8 and reaches a parameter-weighted average of $3.9932$ bits, with full-precision activations.

All methods generate $50{,}000$ images with DDIM-200 \citep{song2021ddim} and
$\eta=0$. They use seeds 40 and 41, the same historical $25{,}000$-image reference, and
\texttt{torch-fidelity} 0.3.0.

\begin{table}[h]
\caption{CelebA-HQ results on $50{,}000$ images. Paired PSNR is the mean of the per-image PSNR values relative to FP32.}
\label{tab:ldm_fid}
\centering

\normalsize
\begin{tabular}{lccc}
\toprule
Method & Weights / activations & FID $\downarrow$ & Paired PSNR $\uparrow$ \\
\midrule
FP32 & FP32 / FP32 & \textbf{20.807} & -- \\
Our method & Mixed ($3.9932$) / A32 & 20.853 & \textbf{34.644} \\
Q-Diffusion & W4 / A32 & 22.626 & 24.060 \\
TFMQ-DM & W4 / A32 & 20.928 & 24.545 \\
\bottomrule
\end{tabular}
\end{table}

Q-Diffusion is the closest allocation-focused control. It shares the calibration pairs and reconstruction backend with our method. Its W4 weights and our mixed-precision weights are reconstructed separately, so this is not a score-only comparison.

TFMQ-DM is a separate end-to-end pipeline designed for diffusion models \citep{huang2024tfmq}. It uses a Temporal Information Block and temporal-feature-aware weight reconstruction. It uses a uniform 4-bit weight allocation in the PSNR comparison, on the same 123 weight tensors as our method.

The four PSNR runs use aligned image indices, seeds, DDIM settings, and output counts. Pairing with historical FP32 and Q-Diffusion outputs is reconstructed from the recorded protocol, with a direct check on $1{,}000$ identical FP32 images. Exact equality of all historical initial noise tensors cannot be verified.

FID \citep{heusel2017fid} compares the generated distribution with the CelebA-HQ reference. Paired PSNR compares each quantized output with the corresponding FP32 output. \revadd{Figure~\ref{fig:ldm_psnr_distribution} shows the PSNR distributions, and Figure~\ref{fig:ldm_qualitative} shows paired generated images.}

\begin{figure}[!t]
\centering
\includegraphics[width=0.86\textwidth]{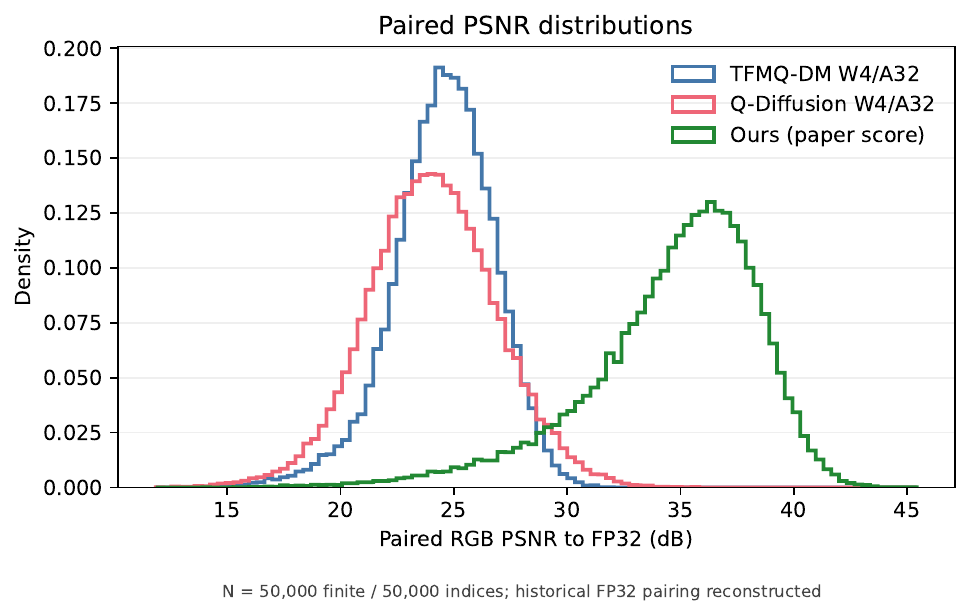}
\caption{Paired PSNR distributions over $50{,}000$ images. This metric measures fidelity to the corresponding FP32 output, not image realism or FID. \revadd{Higher PSNR is better (further right).} Our distribution is shifted towards higher PSNR values, showing closer visual similarity with the paired FP32 outputs.}
\label{fig:ldm_psnr_distribution}
\end{figure}

\begin{figure}[!t]
\centering
\includegraphics[width=\textwidth]{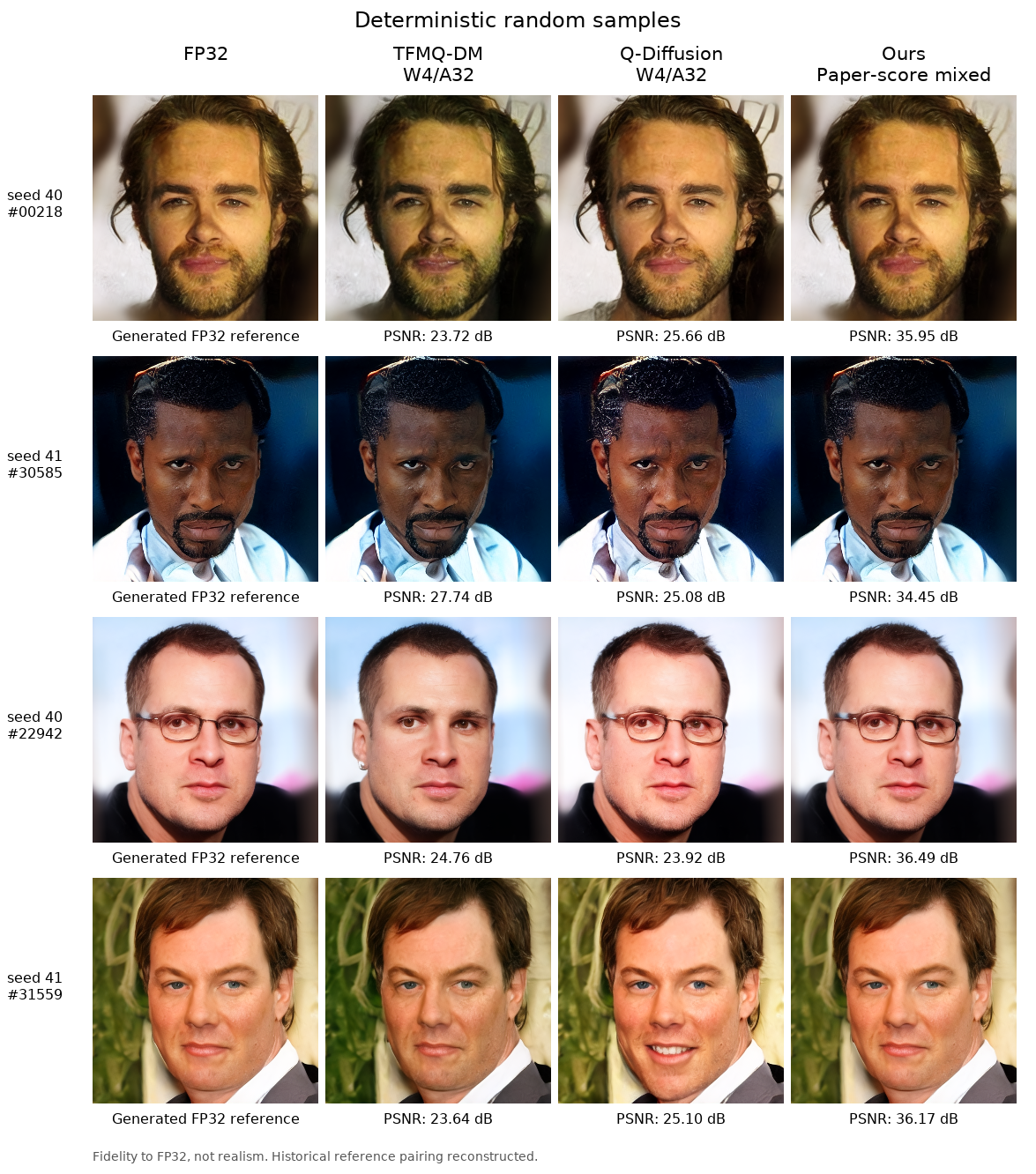}
\caption{Paired CelebA-HQ samples. The four indices were selected by a fixed random rule before the metrics and images were inspected. FP32 is the reference output, not the ground truth. These examples show that our method better preserves the facial features and expressions of the FP32 outputs.}
\label{fig:ldm_qualitative}
\end{figure}

%% file: appendix/additional_drunet.tex
\begin{reviewaddition}
\subsection{Additional natural-image results}
\label{app:bsds_results}

We repeat the DRUNet allocation experiments on natural images from BSDS500~\citep{arbelaez2011contour}. For each of three seeds, we draw $256$ calibration images from the $400$ images in the official training and test splits. We evaluate on the $100$ validation images, using central $256\times256$ RGB crops and Gaussian noise with $\sigma_{\mathrm{in}}=0.1$. Calibration and evaluation images are disjoint. The evaluation images and noise are fixed across methods and seeds. We also report results on the CBSD68 subset of these $100$ images; this is not an independent test set or the native-resolution CBSD68 protocol.

All methods use the same $36$ allocation blocks, symmetric per-output-channel RTN and FP32 intermediate outputs. We compare our method with OMPQ-ORM using W3--W8, and with CLADO using W3/W4/W8. Only calibration inputs receive the
corruptions defined in Section~\ref{sec:common_setup}. OMPQ-ORM's $\beta$ is selected on clean calibration data and then fixed across conditions. The tables report means over the three calibration seeds. Budgets are upper limits, so mean bit-widths are shown explicitly.

Table~\ref{tab:bsds_comparisons} shows that our main W3--W8 configuration remains close to CLADO under clean calibration and outperforms OMPQ-ORM. With the common W3/W4/W8 candidates, our method matches CLADO at $B=4$. CLADO has higher clean PSNR at $B=3.5$. Under both corruptions, our method performs better than CLADO with either candidate set.
Table~\ref{tab:cbsd_subset} shows the same pattern on the CBSD68 subset.

\begin{table}[!htbp]
\color{reviewadded}

\caption{BSDS500 results on $100$ images. Each entry gives PSNR in dB (higher is better), followed by mean bit-width in parentheses. FP32 obtains $31.062$~dB and uniform W4 obtains $28.160$~dB.}
\label{tab:bsds_comparisons}
\centering
\setlength{\tabcolsep}{3pt}
\begin{tabular}{llccc}
\toprule
Method & \shortstack[c]{Candidate\\bit-widths} & Clean & $p=0.5$ shared & $p=0.9$ per-channel \\
\midrule
\multicolumn{5}{c}{$B=3.5$} \\
Our method & 3--8 & 30.415 (3.479) & 30.229 (3.495) & 30.471 (3.494) \\
OMPQ-ORM & 3--8 & 29.038 (3.494) & 29.038 (3.494) & 29.043 (3.494) \\
Our method & 3, 4, 8 & 29.919 (3.470) & 30.007 (3.470) & 30.125 (3.470) \\
CLADO & 3, 4, 8 & 30.580 (3.496) & 25.782 (3.492) & 27.568 (3.500) \\
\midrule
\multicolumn{5}{c}{$B=4$} \\
Our method & 3--8 & 30.901 (3.970) & 30.899 (3.985) & 30.898 (3.916) \\
OMPQ-ORM & 3--8 & 29.551 (3.988) & 29.679 (3.988) & 29.878 (3.988) \\
Our method & 3, 4, 8 & 30.863 (3.968) & 30.863 (3.968) & 30.863 (3.968) \\
CLADO & 3, 4, 8 & 30.876 (3.988) & 26.704 (3.994) & 27.719 (3.969) \\
\bottomrule
\end{tabular}
\end{table}

\begin{table}[!htbp]
\color{reviewadded}

\caption{PSNR on the CBSD68 subset, using the same allocations and mean bit-widths as Table~\ref{tab:bsds_comparisons}. FP32 obtains $31.240$~dB and uniform W4 obtains $28.187$~dB. Higher is better.}
\label{tab:cbsd_subset}
\centering
\setlength{\tabcolsep}{5pt}
\begin{tabular}{llccc}
\toprule
Method & \shortstack[c]{Candidate\\bit-widths} & Clean & $p=0.5$ shared & $p=0.9$ per-channel \\
\midrule
\multicolumn{5}{c}{$B=3.5$} \\
Our method & 3--8 & 30.570 & 30.374 & 30.626 \\
OMPQ-ORM & 3--8 & 29.186 & 29.185 & 29.191 \\
Our method & 3, 4, 8 & 30.060 & 30.150 & 30.270 \\
CLADO & 3, 4, 8 & 30.745 & 25.829 & 27.604 \\
\midrule
\multicolumn{5}{c}{$B=4$} \\
Our method & 3--8 & 31.073 & 31.071 & 31.071 \\
OMPQ-ORM & 3--8 & 29.704 & 29.814 & 30.022 \\
Our method & 3, 4, 8 & 31.035 & 31.035 & 31.035 \\
CLADO & 3, 4, 8 & 31.047 & 26.711 & 27.778 \\
\bottomrule
\end{tabular}
\end{table}

Figures~\ref{fig:bsds_quality} and~\ref{fig:bsds_calibration_effect} show
the quality of our W3--W8 allocations and their response to calibration corruption. Figure~\ref{fig:bsds_examples} gives fixed-random examples at $B=4$, with central zooms in Figure~\ref{fig:bsds_zooms}. The three images are selected independently of image quality and PSNR, and are shared across all displayed methods and conditions.

\begin{figure}[t]
\color{reviewadded}

\centering
\includegraphics[width=\textwidth]{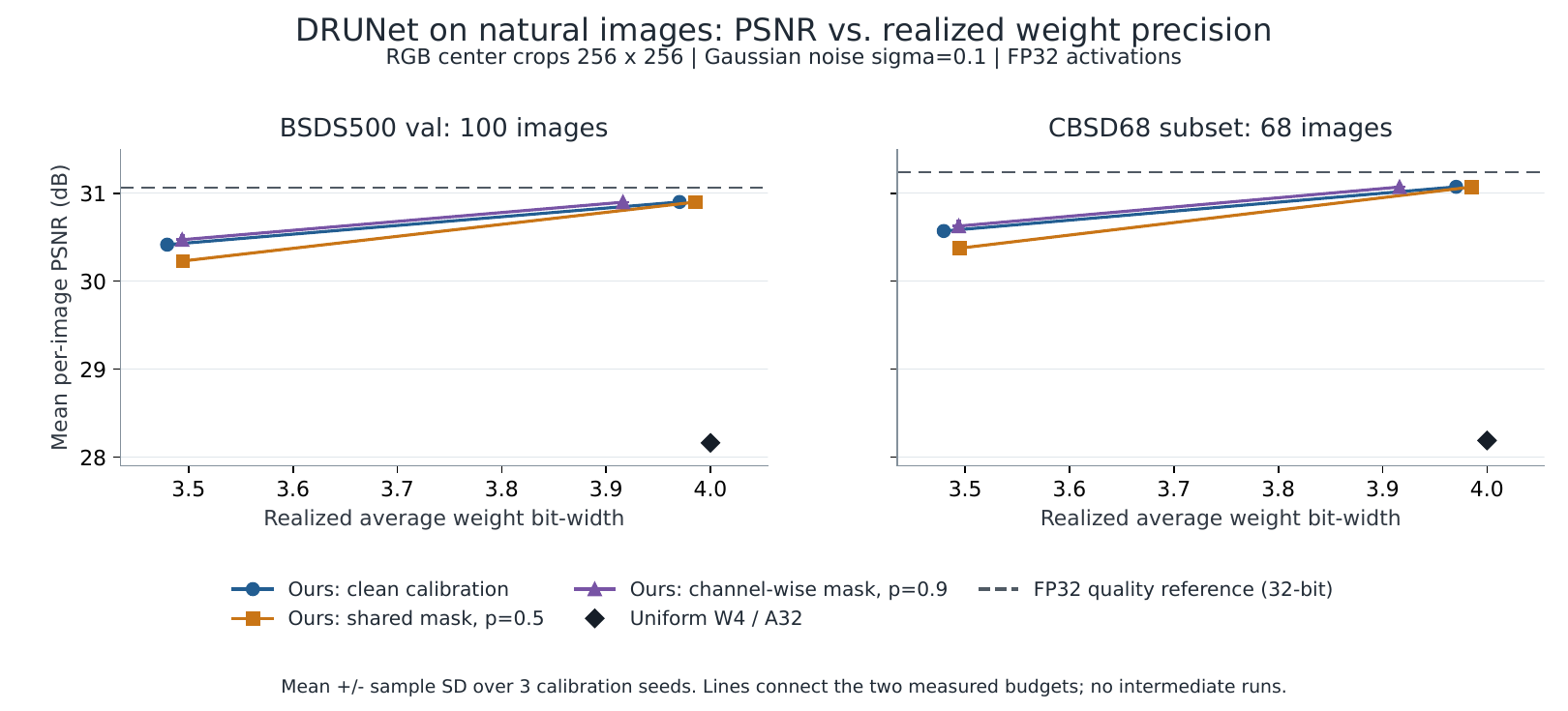}
\caption{DRUNet PSNR on BSDS500 and its CBSD68 subset. Our mixed-precision allocations remain close to FP32 at $B=4$ under all three calibration conditions. Higher PSNR is better. Error bars show the standard deviation over three calibration seeds. Lines connect the two tested budgets.}
\label{fig:bsds_quality}
\end{figure}

\begin{figure}[t]
\color{reviewadded}

\centering
\includegraphics[width=\textwidth]{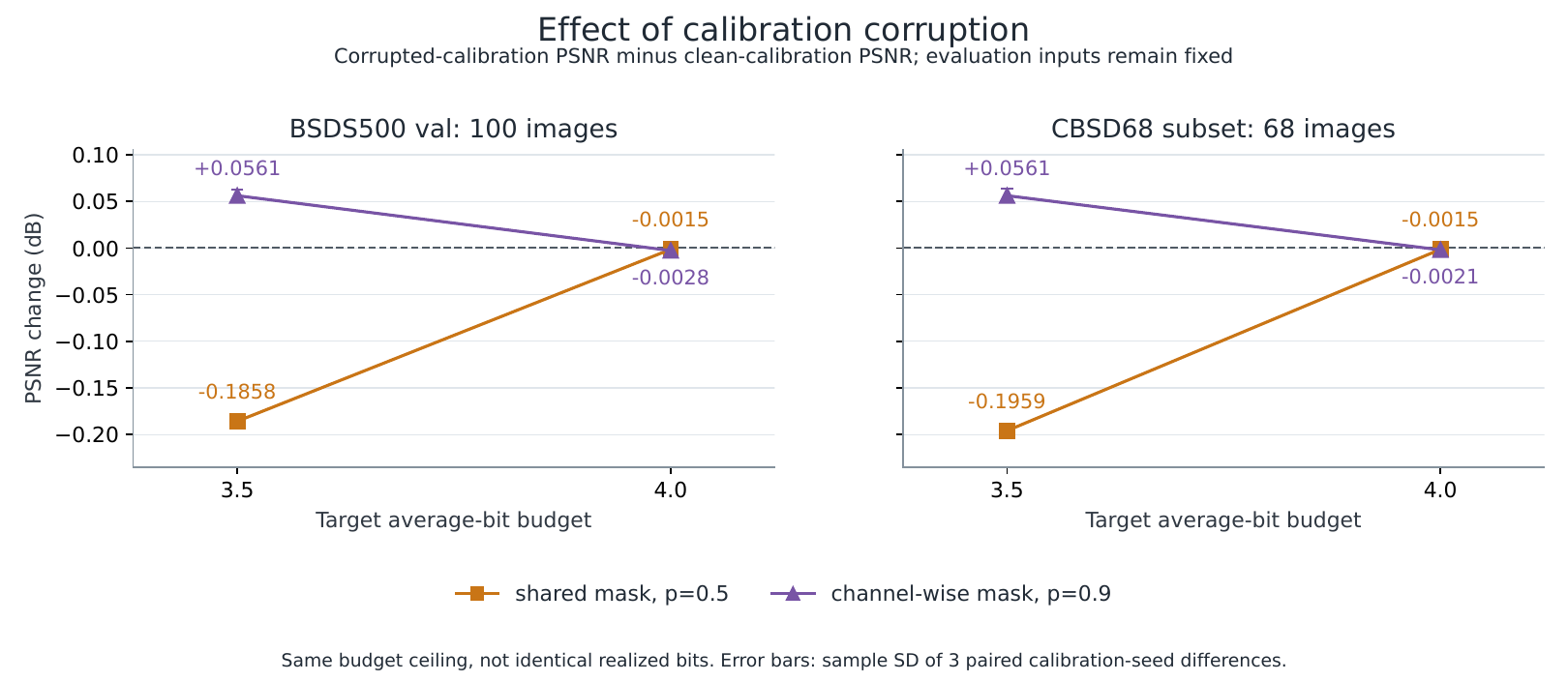}
\caption{PSNR change when calibration inputs are corrupted, with the test set fixed. Values near zero indicate little change in quality. The effect is negligible at $B=4$ and larger at $B=3.5$. Error bars show the standard deviation of the three paired calibration-seed differences.}
\label{fig:bsds_calibration_effect}
\end{figure}

\begin{figure}[t]
\color{reviewadded}

\centering
\includegraphics[page=2,width=\textwidth]{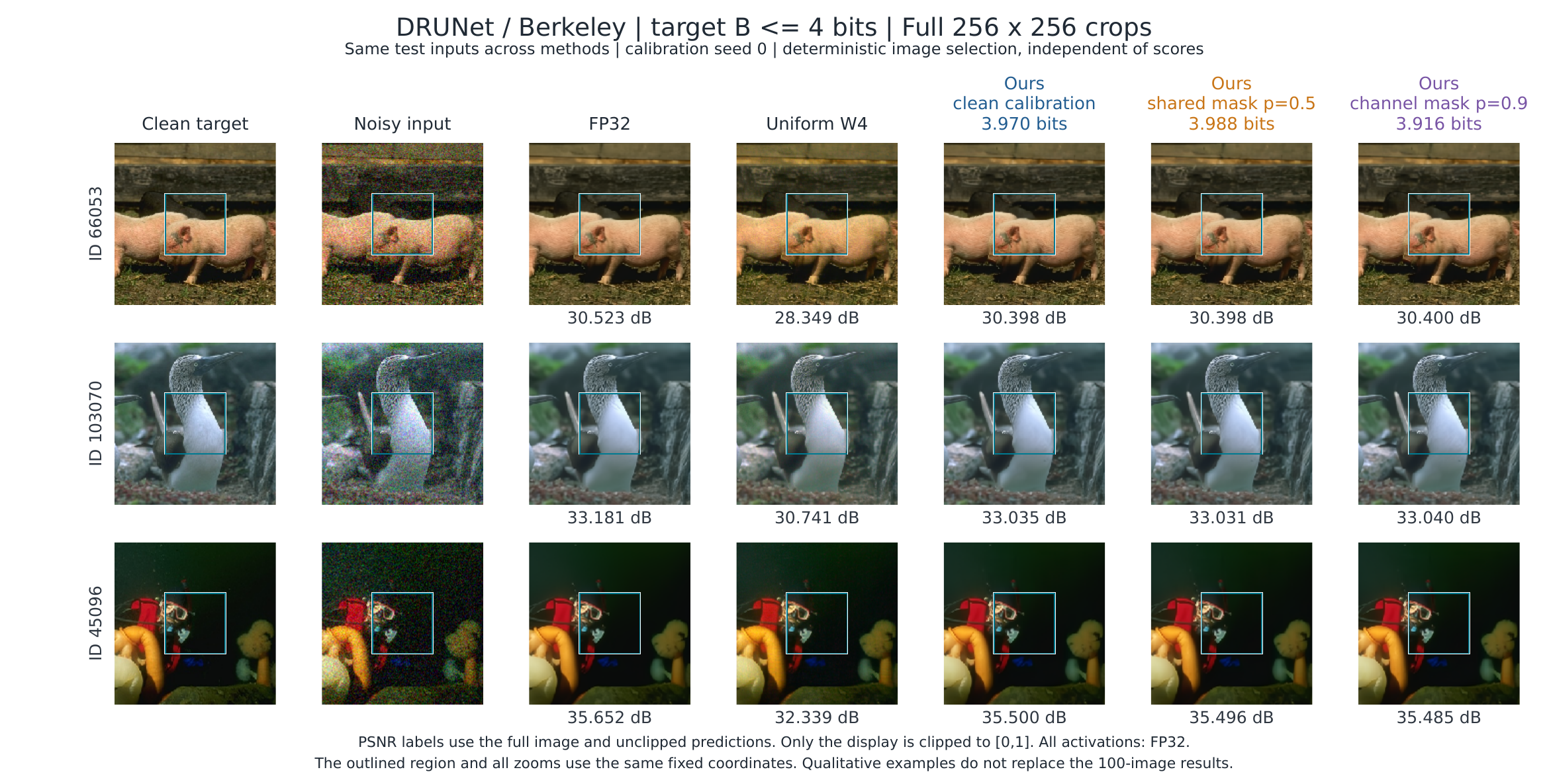}
\caption{Fixed-random BSDS500 examples at $B=4$, using calibration seed 0. The mixed-precision outputs are similar across calibration conditions. PSNR is computed on the full crop before display clipping.}
\label{fig:bsds_examples}
\end{figure}

\begin{figure}[t]
\color{reviewadded}

\centering
\includegraphics[page=4,width=\textwidth]{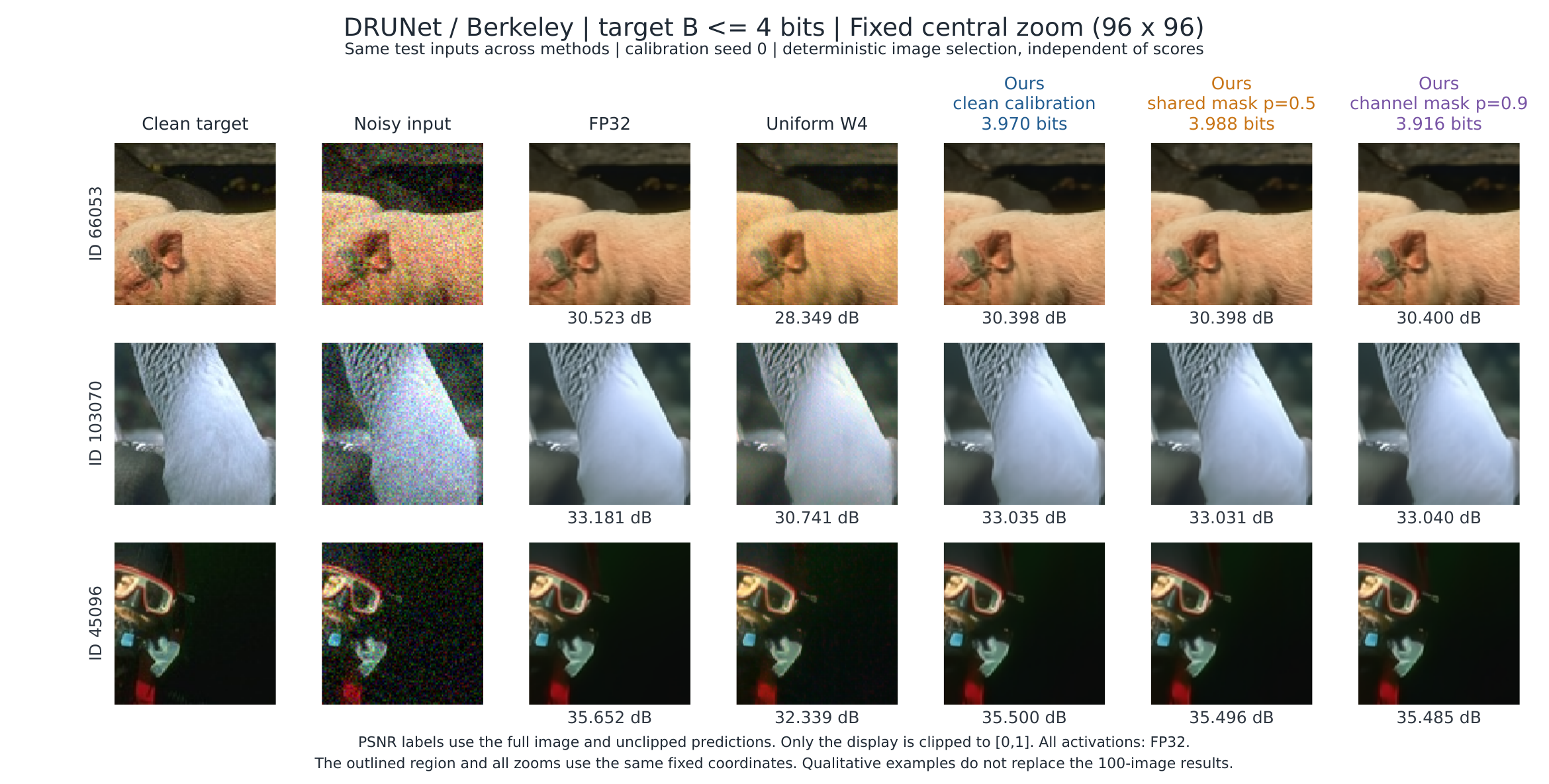}
\caption{Central zooms of Figure~\ref{fig:bsds_examples}. The same $96\times96$ region is shown for every method. PSNR labels refer to the full crops, not the zooms. Our outputs remain similar under the two calibration corruptions.}
\label{fig:bsds_zooms}
\end{figure}

\clearpage
\subsection{Effect of allocation granularity}
\label{app:allocation_granularity}

The low cost of our allocation method allows us to explore finer choices of bit-width. We replace the $36$ blocks by $64$ independent convolutional layers, keeping W3--W8, the local score and the greedy rule unchanged. Both variants reuse the same FP32 calibration passes. Each experiment uses three calibration sets of $256$ images and a fixed test set of $100$ images.

On CelebA, the layerwise allocation improves PSNR by $0.78$~dB at $B=3.5$ and reaches the quality of CLADO (Table~\ref{tab:granularity_celeba}). The CLADO values are those reported in the main comparison; its $36$ groups and W3/W4/W8 candidates differ from our layerwise configuration. The block--layer comparison uses its own paired block control.

\begin{table}[!htbp]
\color{reviewadded}

\caption{Allocation granularity on CelebA with clean calibration. Values are means over three calibration seeds. CLADO is included as a quality reference with its original groups and candidates.}
\label{tab:granularity_celeba}
\centering
\begin{tabular}{llccc}
\toprule
Method & Decisions & \shortstack[c]{Candidate\\bit-widths} & \shortstack[c]{Mean\\bit-width} & PSNR (dB) \\
\midrule
\multicolumn{5}{c}{$B=3.5$} \\
Our method & 36 blocks & 3--8 & 3.482 & 33.298 \\
Our method & 64 layers & 3--8 & 3.494 & 34.078 \\
CLADO & 36 blocks & 3, 4, 8 & 3.494 & 34.101 \\
\midrule
\multicolumn{5}{c}{$B=4$} \\
Our method & 36 blocks & 3--8 & 3.987 & 34.401 \\
Our method & 64 layers & 3--8 & 3.990 & 34.451 \\
CLADO & 36 blocks & 3, 4, 8 & 3.988 & 34.409 \\
\bottomrule
\end{tabular}
\end{table}

On BSDS500, finer allocation improves PSNR at $B=3.5$ in all three calibration conditions, while both granularities give nearly identical quality at $B=4$ (Table~\ref{tab:granularity_bsds}). These results show the benefit of allowing more allocation choices at tight budgets. They compare two grouping choices: grouping also determines which local perturbations are summed before their mean and dispersion are computed.

\begin{table}[!htbp]
\color{reviewadded}

\caption{Allocation granularity on BSDS500, with W3--W8 for both variants. Each entry gives mean PSNR in dB, followed by mean bit-width in parentheses.}
\label{tab:granularity_bsds}
\centering
\setlength{\tabcolsep}{5pt}
\begin{tabular}{lccc}
\toprule
Decisions & Clean & $p=0.5$ shared & $p=0.9$ per-channel \\
\midrule
\multicolumn{4}{c}{$B=3.5$} \\
36 blocks & 30.415 (3.479) & 30.229 (3.495) & 30.471 (3.494) \\
64 layers & 30.656 (3.498) & 30.581 (3.494) & 30.658 (3.498) \\
\midrule
\multicolumn{4}{c}{$B=4$} \\
36 blocks & 30.901 (3.970) & 30.899 (3.985) & 30.898 (3.916) \\
64 layers & 30.904 (3.975) & 30.903 (3.993) & 30.899 (3.997) \\
\bottomrule
\end{tabular}
\end{table}

\end{reviewaddition}